\pdfoutput=1
\documentclass{article}
\usepackage[utf8]{inputenc}
\usepackage[english]{babel}
\usepackage[hyphens]{url}
\usepackage{hyperref}
\usepackage{natbib}
\usepackage{subcaption}
\usepackage{newfloat}
\usepackage{listings}

\makeatother

\usepackage{graphicx} 
\usepackage{geometry}
\usepackage[nottoc,numbib]{tocbibind}
\hypersetup{
    breaklinks=true,
    colorlinks=true, 
    linktoc=all,     
    linkcolor=black,  
}
\usepackage{amsmath}
\usepackage{amsfonts} 
\usepackage{mathtools}
\usepackage{amsthm}
\usepackage{thmtools,thm-restate}
\usepackage{dsfont}

\DeclareMathOperator{\softm}{Softmax}
\newcommand\p[2]{\frac{\partial #1}{\partial #2}}
\newcommand{\const}[2]{\frac{\sigma_{#1}}{\sqrt{d_{#2}}}}
\newcommand{\inR}[1]{\in \mathbb{R}^{#1}}
\DeclareMathOperator{\concat}{concat}
\DeclareMathOperator{\vect}{vec}
\DeclareMathOperator{\diag}{diag}
\DeclareMathOperator{\dropout}{dropout}

\DeclareMathOperator{\tr}{tr}
\DeclareMathOperator{\bm}{\mathrm{batchm}}

\newcommand{\WW}[1]{\overset{\scriptscriptstyle 1}{W} \vphantom{W}^{#1}}
\newcommand{\WWW}[1]{\overset{\scriptscriptstyle 2}{W} \vphantom{W}^{#1}}
\newcommand{\SSigma}[1]{\overset{\scriptscriptstyle 1}{\Sigma} \vphantom{\sigma}^{#1}}
\newcommand{\SSSigma}[1]{\overset{\scriptscriptstyle 2}{\Sigma} \vphantom{\sigma}^{#1}}
\newcommand\pv[2]{\frac{\partial\vect(#1)}{\partial\vect(#2)}}

\newcommand{\Exp}[1]{\mathbb{E}\bigl[ #1 \bigr]}

\DeclareTextFontCommand{\emph}{\boldmath\bfseries}
\newtheorem{theorem}{Theorem}[section]

\newtheorem{lemma}{Lemma}[section]
\newtheorem{corollary}{Corollary}[section]

\newenvironment{sproof}{%
\proof}{\endproof}

\allowdisplaybreaks

\title{Infinite Width Graph Neural Networks for Node Regression/ Classification}
\author{Yunus Cobanoglu\footnote{\href{mailto:yunus.cobanoglu@yahoo.de}{yunus.cobanoglu@yahoo.de}, work done as part of the Master's Thesis at the Department of Mathematics, Technical University of Munich}}

\date{} 

\begin{document}

\maketitle
\begin{abstract} \noindent
This work analyzes Graph Neural Networks, a generalization of Fully-Connected Deep Neural Nets on graph structured data, when their width, that is the number of nodes in each fully-connected layer is increasing to infinity. Infinite Width Neural Networks are connecting Deep Learning to \emph{Gaussian Processes} and \emph{Kernels}, both Machine Learning Frameworks with well-established theoretical foundations. Gaussian Processes and Kernels have less hyperparameters than Neural Networks and can be used for uncertainty estimation, making them more user-friendly for applications. This works extends the increasing amount of research connecting Gaussian Processes and Kernels to Neural Networks. 
The Kernel and Gaussian Process closed forms are derived for a variety of architectures, namely the standard Graph Neural Network, the Graph Neural Network with Skip-Concatenate Connections and the Graph Attention Neural Network. All architectures are evaluated on a variety of datasets on the task of transductive Node Regression and Classification. Additionally, a Spectral Sparsification method known as \textit{Effective Resistance} is used to improve runtime and memory requirements.
Extending the setting to inductive graph learning tasks (Graph Regression/ Classification) is straightforward and is briefly discussed in \ref{ind}.
\end{abstract}
\clearpage
\tableofcontents{}
\clearpage

\section{Introduction}\label{intro}
Graph Neural Networks (GNNs), introduced in the paper \citep{kipf2017semisupervised} demonstrated their effectiveness in tasks involving graph-structured data. GNNs have become a staple tool for data scientists and a prominent area of research in Machine Learning. In recent research, GNNs have been extended to different architectures such as Graph Attention Neural Networks \citep{veličković2018graph},
Transformer Graph Neural Networks \citep{shi2021masked} and to unsupervised learning tasks with a Graph (Variational) Autoencoder \citep{kipf2016variational}. \\
Another very active area of research are infinite width Neural Networks, also called Neural Tangent Kernels (NTK), introduced in the paper \citep{jacot2020neural}. NTK link infinite width Fully-Connected Deep Neural Nets (FCN) and Kernels. NTK theory has been expanded to various architectures \citep{yang2} and used to analyze theoretical properties like Generalization and Optimization \citep{arora2019exact, provably} and the importance of appropriate random initialization for successful training \citep{xiao2020disentangling, Seleznova}. \\
This work extends infinite width Neural Network theory to the setting of Graph Neural Networks. The main contributions\footnote{Code is available at 
\url{https://github.com/yCobanoglu/infinite-width-gnns} 
}
are:
closed form expressions for the GNN Gaussian Process (GNNGP) and Graph Neural Tangent Kernel (GNTK) for three different architectures on the task of Node Regression/ Classification:  the vanilla GNN, the Skip-Concatenate GNN and the Graph Attention Neural Network and 
evaluation of the GNNGP and GNTK to their Neural Network counterparts on a variety of datasets, including applying \textit{Effective Resistance} as a Spectral Graph Sparsification Method to improve memory and runtime requirements.

\subsection{Notation}
Superscript for matrices (i.e. $W^{L}$) are identifiers and not matrix powers. $I_{n} \inR{n \times n}$ is the identitiy Matrix.
We define the operator batchmultiply; $\bm(X, Y)_{ij} = \langle X, Y_{IJ} \rangle_{F}$ with $X \in \mathbb{R}^{n \times n} $ and $Y \in \mathbb{R}^{zn \times zn}$ for some integer $z$.
So $\bm(\cdot, \cdot)$ is taking the Frobenius Inner Product of $X$ and block $IJ$ of size $n \times n$ of $Y$ and $\bm(X, Y) \in \mathbb{R}^{z \times z}$. Concatenating vectors/matrices horizontally is denoted $\concat(\cdot, \cdot)$. The Hadamard Product is denoted with $\odot$. Convergence in probability is denoted as $\overset{P}{\longrightarrow}$.

\subsection{Graph Neural Networks generalize Fully-Connected Deep Nets}
GNNs are a powerful generalization of Fully-Connected Nets (FCN). GNNs are extending the FCNs by making use of dependent data samples, naturally incorporating unlabeled data, referred to as semi-supervised learning and allowing for transductive learning tasks such as Node Regression and Classification as well as inductive tasks like Graph Regression/ Classification.
We will start by establishing this connection between FCNs and GNNs.
A three layer FCN can be defined as follows:
\begin{align}
    f(x) &= W^{3} \sigma(W^{2}( \sigma (W^{1} x))  \inR{d_{3}} \text{ with $x  \inR{d_{0}}$ and weights $W^{l} \inR{d_{l} \times d_{l-1}}$.}
\end{align}
The single data point $x$ is a column vector of a data matrix $X\inR{d_{0} \times n}$.
Now defining a GNN, as in \cite{kipf2017semisupervised}. Given a graph adjacency matrix $A \in \mathbb{R}^{n \times n}$ and the data $X^{T} \in \mathbb{R}^{n \times d_{0}}$, the GNN is:
\begin{align}\label{axw}
g(X) = A\sigma(A\sigma(AX^{T}W^{1})W^{2})W^{3}  \inR{n  \times d_{L}} 
\end{align}
with $W^{l} \inR{d_{l-1} \times d_{l}}$.
There are different options for the adjacency matrix $A$. The adjacency matrix can be the 0-1 adjacency matrix of an undirected graph, denoted as $A_{0\textbf{-}1}$. $A$ can be $A_{0\textbf{-}1} + I$, which is the 0-1 adjacency matrix with self loops. $A$ can be the Laplacian $L = D - A$ with $D$ being the degree matrix, $D_{ii} = \sum_{ij} A_{ij}$. One can use row, column or row \& column normalization on the Laplacian.
$A$ can be Degree normalized according to \cite{kipf2017semisupervised} which is $A= (D+I)^{-\frac{1}{2}} (A + I)  (D+I)^{-\frac{1}{2}}$, with D being defined as above.
If we set $A$ as the Identity matrix $I$ and take the transpose of $g(X)$, we end up with:
\begin{align}
    g(X)^{T} =  W^{3^{T}} \sigma(W^{{2^{T}}}( \sigma (W^{1^{T}} X^{T})) \inR{d_{3} \times n} 
\end{align}
$g(X)^{T}$ is equivalent to a FCN $f(X)$ where we pass as input all the (training and test) data $X$ instead of a single data point $x$. 
In the case of Node Regression, the training loss is minimized as follows:
\begin{align}
    \underset{W}{\operatorname{argmin}}(||(Y^{T} - g(X)^{T}) I_{train} ||_{F}^{2}) 
    \text{  with $I_{train} = \diag([0,0,1,0,1,...,0,1])$}
\end{align}
$I_{train} \inR{n \times n}$ consists of randomly selected diagonal 0,1 entries which correspond to the training samples. $I_{test}$ is defined as $I_{test}= I_{n} - I_{train}$. We have labels Y $\inR{n}$ and $I_{train}$ is randomly selecting columns of $g(X)^{T}$ over which the training loss is minimized. Columns in $g(X)^{T}$ correspond to individual training samples when the GNN reduces to a FCN (i.e. $g(X)^{T} = f(X)$) or individual nodes in when $g(X)^{T}$ is a proper GNN, in which case we will call this setup Node Regression/ Classification.
The task of Node Regression/ Classification for the GNN with the Identity matrix is equivalent to training a FCN with Gradient descent.
Stochastic Gradient Descent (SGD) can be generalized to GNNs straightforwardly, too.
\begin{align}
    g(X)^{T} =  W^{3^{T}} \sigma(W^{{2^{T}}}( \sigma (W^{1^{T}} X^{T} \dropout(I))) 
\end{align}
Sampling from $I$ without replacement some subset with size k is equivalent to minimizing our objective using SGD with batch size k. This can be emulated using dropout.
\begin{align}
    g(X)^{T} =  W^{3^{T}} \sigma(W^{{2^{T}}}( \sigma (W^{1^{T}} X^{T} \dropout(A))A)A 
\end{align}
Dropout is commonly applied to both the adjacency matrix $A$ and weights in GNN layers. Recognizing GNNs as generalizations of FCNs, the same applies to GNNGP and NNGP, as well as GNTK and NTK. Substituting the Identity matrix $I$ for $A$ in GNNGP and GNTK reduces to the NNGP and the NTK. \\
\\
Node Regression/Classification with GNNs involves transductive learning, where the model can incorporate test data during training, see \citep[transductive learning Section 10.1, Setting 1]{Vapnik06} for a formal definition. The expression $AXW$ in \ref{axw} is nothing else $(AXW)_{i}= \sum_{k \in \text{Neighbors}(i)} x_{k} W$ (for standard the 0-1 adjacency matrix $A$). This summation includes neighboring nodes that could be part of the test set, making it a transductive learning task. Unlabeled nodes, similar to test nodes, can be used during training, hence Node Regression/Classification is sometimes called semi-supervised learning.

\section{Related Work}\label{related}
The GNTK formulation for inductive Graph Classification/Regression was initially established in \cite{DuHSPWX19}. Note that any Node Classification GNN architecture can be used for Graph Classification by simply incorporating an aggregation layer (such as Max Pooling or averaging) as the final layer (see Section \ref{ind}). Consequently, the GNTK for Node Classification serves as a fundamental building block for constructing the GNTK for Graph Classification.
In the derivation by \cite{DuHSPWX19}, the GNTK is established for inductive learning tasks related to Graph Classification/Regression. Their BLOCK Operation, which corresponds to the GNTK for Node Classification/Regression, differs from our approach in several aspects. Notably, their derivation relies on the assumption of 0-1 adjacency with added self-loops. Moreover, since the GNTK formula is derived only for two elements of the Kernel matrix $\Theta_{ij}$, the adjacency information is implicitly encoded in the final expression. Therefore, the closed form is not stated in the most general way. Consequently, no conclusive insights can be drawn regarding the impact on the final expression of the NTK, such as the positive definiteness of the Kernel. In contrast, the GNTK formula derived in our work imposes no assumptions on the adjacency matrix A. The derivation is done using basic tools from Linear Algebra. Notably, \cite{DuHSPWX19} focuses exclusively on inductive learning and does not conduct experiments related to Node Classification/Regression tasks. Furthermore, there are no simulations demonstrating whether the GNTK aligns with NTK Theory for GNNs. Simulations presented in Section \ref{sims} reveal that NTK assumptions do not hold when using the 0-1 adjacency but are satisfied when employing the Kipf \& Welling normalized adjacency.
The GNTK formulas in this work are also easily extended to the inductive learning setting, see \ref{ind}. \\
\\
\cite{niu2023} derived the GNNGP expression for the standard GNN with the 0-1 adjacency matrix known as GIN \citep{xu2019powerful}, the Kipf \& Welling normalized adjacency as defined in section \ref{intro} and two variants of GNNS with Skip-Connections, namely the GraphSage \citep{hamilton2018inductive} and GCNII \citep{pmlr-v119-chen20v}. These Skip-Connections differ from our Skip-Concatenate, as they involve summation, while ours entails concatenation. 
Work on GNTK for Node Classification/Regression is explored by \citep{gntk1}, although their derivation of the GNTK is incorrect.
\begin{theorem}[\textbf{incorrect}, from {\cite[Theorem 1 (NTK for Vanilla GCN)]{gntk1}}]\label{th: NTK for GCN}
For the vanilla GCN defined in \eqref{axw} (with the difference that S is the Adjacency Matrix and not A),
the NTK $\Theta$ is given by
\begin{align}
      \Theta &= \biggr[ \sum_{i=1}^{d+1} \Sigma_i \odot ( SS^T)^{ \odot (d+1-i)} \odot   \bigl( \odot^{d+1-i}_{j=i} \dot{E}_j \bigr) \biggr] \\
       &\odot \, \underset{f\sim \mathcal{N}(0,\Sigma_d)}{\mathbb{E}} [\dot{\Phi}( f ) \dot{\Phi} (f)^T]
\end{align}
Here $\Sigma_i\inR{n \times n}$ is the co-variance between nodes of the layer $f_i$, and is given by $\Sigma_1 := SXX^TS^T$, $\Sigma_i := S E_{i-1} S^T$ with $E_i:= c_\sigma \underset{f \sim \mathcal{N}(0, \Sigma_{i})}{\mathbb{E}}[\sigma(f) \sigma(f)^T]$ and $\dot{E}_i:= c_\sigma \underset{f \sim \mathcal{N}(0, \Sigma_{i})}{\mathbb{E}}[\dot{\sigma}(f) \dot{\sigma}(f)^T]$. $\Phi(x)  \approx \frac{1}{1+ e^{-x}}$ (i.e Sigmoid function).
\end{theorem} \noindent
Replacing $S$ with the Identity matrix in the GNTK formula does not yield the NTK for Fully-Connected Deep Nets. The final result will be a Kernel matrix with all entries zero other than the entries on the diagonal which can not be correct. For comparison the correct GNTK is given in \ref{th:gntk}, while the NTK expression is given in \ref{ntk}).
\cite{sabanayagam2023representation} used this wrong GNTK formula in subsequent work to investigate representational properties. The derivation of the Graph Attention NTK and GP relies on results from \cite{pmlr-v119-hron20a, yang1}, as the NTK and GP for Attention Neural Networks closely relate to the Graph Attention Neural Network.
\section{Infinite Width Graph Neural Networks}\label{mains}
Before stating the main results of this work we are going to summarize the results on Infinite Width Fully-Connected Networks as our derivations for Infinite Width Graph Neural Networks will be reduced to the Fully-Connected derivations. 
\subsection{Recap: Fully-Connected Deep Nets}\label{recap}
In Lemma \ref{ntk:lemma} we prove how the output of an FCN during training with $l_{2}$ loss with Gradient Flow (i.e. infinitesimally step size for Gradient Descent) can be described by the following differential equation:
{\small
\begin{gather} \label{diffeq}
    \frac{du(\theta(t))}{dt} =\p{u(\theta(t))}{\theta}  \left( \p{u(\theta(t))}{\theta}  \right)^{T} 
    \left( u(\theta(t)) - \vect(Y) \right)
\end{gather}
}%
with $u(\theta(t)):=\vect(F(\theta(t), X))$
 and $\theta$ being all trainable parameters flattened to a vector. This equation nothing is else then Kernel Regression under Gradient Flow.
When all widths of the hidden Layers of $F(\theta(t), X))$ go to infinity, $\p{u(\theta(t))}{\theta}  \left( \p{u(\theta(t))}{\theta} \right)^{T}= \Theta$ has a closed form expression which stays constant during training and the equation reduces to Kernel Regression. $\Theta$ is referred to as the NTK.
For increased width during training we can observe 1) the training loss approaches zero 2) the weights of the GNN stay close to its initialization (measured in relative Frobenius norm) and the empirical NTK, (i.e. $\p{u(\theta(t))}{\theta}  \left( \p{u(\theta(t))}{\theta} \right)^{T}$) stays close to its initialization during training. All these observations could be proven in the case of Fully-Connected Deep Nets \citep{arora2019exact, jacot2020neural}.
The proof of 1) consists of showing that the empirical NTK stays positive definite during training \citep[Lemma 3.3]{du2018gradient}. Proof of 2) consists of first showing that during initialization the empirical NTK is close to the closed form formula (which is derived by letting the widths of each layer go to infinity). After that one is left with showing that during training the weights do not change by a lot using Gradient Flow Dynamics and interpreting training as a Perturbation on the Networks' weights. For a detailed proof, the reader is referred to \cite{arora2019exact}.
Simulations in \ref{sims} confirm 1), 2) and 3) for wide widths GNNs for Node Regression/ Classification using the Kipf \& Welling Normalized Adjacency Matrix and the GAT Model with the 0-1 Adjacency with self loops.
In the following sections we are stating the closed form expressions for the NNGP and the NTK for Fully-Connected Deep Neural Nets.
We are going to formally define FCNs and rederive the NNGP and NTK closed form expressions.
The Weight Matrices are $W^{h} \in \mathbb{R}^{d_{h} \times d_{h-1}}$ and the data Matrix is $X \in \mathbb{R}^{d_{0} \times n }$ where  $d_{0}$ is the feature dimension of the data and $n$ is the number of data samples. $B$ denotes the rank-one bias matrix.
$\sigma$ is a non-linearity (e.g. Relu) and is applied elementwise.

\begin{restatable}[Fully-connected Deep Neural Nets]{definition}
{deffcn}\label{fcn-ef}
\begin{align}
F^{h}(X) &= \const{w}{h-1} W^{h} G^{h-1}(X) + \sigma_{b} B^{h}  \in \mathbb{R}^{d_{h} \times n} \\
\text{with } B^{h}&= b^{h} \otimes \mathbf{1}_{n}^{T} \text{ and } b^{h} \in \mathbb{R}^{d_{h}} \\ 
G^{h}(X) &= \sigma(F^{h}(X)) \\
G^{0} &= X \inR{d_{0} \times n}
\end{align}
The Weights are initialized as $W^{h} \sim N(0,1)$ and $b^{h} \sim N(0, 1)$. \\
The final Network is $F^{L}(X) = \frac{\sigma_{w}}{\sqrt{d_{L-1}}} W^{L} G^{L-1}(X) + \sigma_{b} B^{L} \in \mathbb{R}^{d_{L} \times n}$
\end{restatable}

\subsubsection{Gaussian Process}
The FCN and NNGP equivalence first was shown by \citep{Neal1996} for a one Layer Neural Network $F^{1}(X)$ and later proven for Infinite Width by \citep{lee2018deep, gp1}. The simplest form of this proof requires the width of each layer of the FCN to go to infinity successively. \citep{gp1} could refine the proof by requiring all layers to go to infinity simultaneously. 

\begin{restatable}[Neural Network GP]{theorem}{thnngp}\label{th:nngp}\ \\
If all the Weight dimensions $d_{L-1},d_{L-1}, ..., d_{1}$ other than the input and the output dimensions of a Neural Network $F(X)^{L}$ successively go to infinity then $\vect(F(X)^{L}) \sim$ GP(0, $\Lambda^{L-1}\otimes I_{d_{L}}$ ), with \\
\begin{gather}
    \Lambda^{L} = \sigma^{2}_{w}\begin{pmatrix} 
    \mathbb{E} \bigl[ \sigma(u_{1}) \sigma(u_{1}) \bigr] & \mathbb{E} \bigl[ \sigma(u_{1}) \sigma(u_{2}) \bigr] \dots  & \mathbb{E} \bigl[ \sigma(u_{1}) \sigma(u_{n}) \bigr]\\
    \vdots & \ddots & \vdots\\
    \mathbb{E} \bigl[ \sigma(u_{n}) \sigma(u_{1}) \bigr] & \dots  & \mathbb{E} \bigl[ \sigma(u_{n}) \sigma(u_{n}) \bigr]
    \end{pmatrix} + \sigma_{b}^{2}\\[10pt]
    \text{ with } u \sim GP(0, \Lambda^{L-1}) \\
    \intertext{$\Lambda^{L} $ can be written as the outer product of the random vector $u$} \Lambda^{L}= \sigma^{2}_{w}\mathbb{E} \left[ \sigma(u) \sigma(u^{T}) \right] +\sigma_{b}^{2}\\
    \intertext{and the base case concludes the theorem}
    \Lambda^{0} = \frac{\sigma^{2}_{w}}{d_{0}} X^{T}X   + \sigma_{b}^{2}
\end{gather}
\end{restatable} \noindent
The Integral $\mathbb{E}\left[ \sigma(u) \sigma(u^{T}) \right] $ with $u$ being a multivariate Gaussian vector can be calculated in closed form for certain activation functions like Relu. 
\cite{asd123} present an overview with closed form expressions for activation functions like Relu, Error function, Leaky Relu, Exponential, RBF, and others.
For a list of efficient implementation see  \cite{neuraltangents2020}. The proof of Theorem \ref{th:nngp} is given in \ref{th:nngpproof}.
\begin{sproof}\renewcommand{\qedsymbol}{}\ 
The proof is done by induction on the depth $l$ for $F^{l}(X)$. The Base Case is established on the fact that the output of $\vect(Y^{1})=\vect(F(X)^{1})$ is a weighted sum of zero mean multivariate Gaussian. The induction step is based on the fact that the output of $\vect(F^{l+1}(X))$ conditioned on $\vect(G^{l}(X))$ is again a Zero Mean Multivariate Gaussian. When the output of the previous layer goes to infinity, so $d_{l} \rightarrow \infty$, we will have an infinite sum which will converge to its mean by the Law of Large Numbers.
\end{sproof}
\clearpage

\subsubsection{Neural Tangent Kernel}\label{ntk}
\begin{restatable}[NTK]{theorem}{thnntk}\label{th:ntk}\ \\
If all the Weight dimensions $d_{L-1},d_{L-1}, ..., d_{1}$ other than the input and the output dimensions of a Neural Network $F(X)^{L}$  successively go to infinity, the NTK has a closed form, namely $\Theta^{L} \otimes I_{d_{l}}$ with
\begin{gather}
\Theta^{L} = \Lambda^{L-1} +  \left( \dot{\Lambda}^{L-1}  \odot \Theta^{L-1} \right) \\
    \Lambda^{L} = \sigma^{2}_{w}\mathbb{E} \left[ \sigma(u) \sigma(u^{T}) \right] +\sigma_{b}^{2} \\
        \dot{\Lambda}^{L} = \sigma^{2}_{w}\mathbb{E} \left[ \dot{\sigma}(u) \dot{\sigma}(u^{T}) \right] +\sigma_{b}^{2} 
        \intertext{\centering with $u \sim N(0, \Lambda^{L-1})$}
        \Theta^{1} = \Lambda^{0} = \frac{\sigma^{2}_{w}}{d_{0}} X^{T}X + \sigma^{2}_{b}
\end{gather}
\end{restatable} \noindent
The proof of Theorem \ref{th:ntk} is given in \ref{th:ntkproof}.
\begin{sproof}\renewcommand{\qedsymbol}{}
    The proof is based on Induction on the Network Depth $l$ and uses the NNGP derivations from above in addition with the Law of Large Numbers.
\end{sproof} \noindent
The NTK Closed Form can be simplified by expanding the recursion and reordering the terms.
\begin{corollary}[Non-recursive NTK Formula {\citep[Equation 9][]{arora2019exact}}]
\begin{align*}
   & \dot{\Lambda}^{L+1} :=  \mathds{1}_{n} \\
   & \Theta^{L} =    \sum_{h=2}^{L+1} \left( \Lambda^{h-1}  \odot  ( \dot{\Lambda}^{h} \odot \dot{\Lambda}^{h+1} \odot \cdots \odot   \dot{\Lambda}^{L+1}  )  \right) \otimes I_{d_{L}} \label{eq:29}
\end{align*}
\end{corollary}
\clearpage
\subsection{Graph Neural Networks}
We define the Graph Neural Net recursively similar to the FCN previously.
The matrix A is commonly the adjacency matrix with $A \in \mathbb{R}^{n \times n}$, but for the GP and NTK no assumptions about A are used.

\begin{restatable}[Graph Neural Network]{definition}{defgnn}\label{defi:gnn}
\begin{align}
F^{h}(X) &= \biggl( \frac{\sigma_{w}}{\sqrt{d_{h-1}}} W^{h} G^{h-1}(X) + \sigma_{b} B^{h}\biggr) A^{T}  \in \mathbb{R}^{d_{h} \times n} \\
\text{with } B^{h}&= b^{h} \otimes \mathbf{1}_{n}^{T} \text{ and } b^{h} \in \mathbb{R}^{d_{h}} \\ 
G^{h}(X) &= \sigma(F^{h}(X)) \\
G^{0} &= X \inR{d_{0} \times n}
\end{align}
The Weights are initialized as $W^{h} \sim N(0,1)$ and $b^{h} \sim N(0, 1)$ \\
The final Network of depth $L$ is \\
$F^{L}(X) = \biggl( \frac{\sigma_{w}}{\sqrt{d_{L-1}}} W^{L} G^{L-1}(X) + \sigma_{b} B^{L} \biggr) A^{T} \in \mathbb{R}^{d_{L} \times n}$ \\
\end{restatable}

\subsubsection{Graph Neural Network Gaussian Process}
\begin{restatable}[GNN Gaussian Process]{theorem}{thngnngp}\label{th:gnngp} 
If all the weight dimensions of the hidden layers $d_{L-1},d_{L-1}, ..., d_{1}$  of a GNN, $F(X)^{L}$ successively go to infinity, the NTK has a closed form expression $\Theta^{L} \otimes I_{d_{L}}$, with
\begin{gather}
\Theta^{L} =  A \left( \Lambda^{L-1} +  \left( \dot{\Lambda}^{L-1}  \odot \Theta^{L-1} 
\right) \right)  A^{T} \text{ with } \\
    \Lambda^{L} = \sigma^{2}_{w}\mathbb{E} \left[ \sigma(u) \sigma(u^{T}) \right] +\sigma_{b}^{2}  \\
        \dot{\Lambda}^{L} = \sigma^{2}_{w}\mathbb{E} \left[ \dot{\sigma}(u) \dot{\sigma}(u^{T}) \right] +\sigma_{b}^{2}
        \intertext{ with $u \sim N(0, A\Lambda^{L-1}A^{T})$} 
        \Theta^{1} = \Lambda^{0} = \sigma^{2}_{w} X^{T}X + \sigma^{2}_{b}
\end{gather}
\end{restatable} \noindent
\begin{sproof}\renewcommand{\qedsymbol}{}
Applying a property of the Kronecker Product $(A \otimes B)(C \otimes D) = (AC \otimes BD)$ and the bilinearity of the Covariance reduces the proof to the Neural Network Gaussian Process (NNGP) proof (see \ref{th:nngp}).
\end{sproof} \noindent
The proof of Theorem \ref{th:gnngp} is given in \ref{th:gnngp1}. 
The expression $\mathbb{E}\left[ \sigma(u) \sigma(u^{T}) \right] $ with $u$ being a multivariate Gaussian vector can be calculated  in closed form for certain activation functions (see \ref{th:nngp} for a detailed discussion).

\subsubsection{Graph Neural Tangent Kernel}
\begin{restatable}[GNN NTK]{theorem}{thngntk}\label{th:gntk}\ \\
If all the weight dimensions of the hidden layers $d_{L-1},d_{L-1}, ..., d_{1}$  of a GNN, $F(X)^{L}$ successively go to infinity, the NTK has a closed form expression $\Theta^{L} \otimes I_{d_{L}}$, with
\begin{gather}
\Theta^{L} =  A \left( \Lambda^{L-1} +  \left( \dot{\Lambda}^{L-1}  \odot \Theta^{L-1} 
\right) \right)  A^{T} \text{ with } \\
    \Lambda^{L} = \sigma^{2}_{w}\mathbb{E} \left[ \sigma(u) \sigma(u^{T}) \right] +\sigma_{b}^{2}  \\
        \dot{\Lambda}^{L} = \sigma^{2}_{w}\mathbb{E} \left[ \dot{\sigma}(u) \dot{\sigma}(u^{T}) \right] +\sigma_{b}^{2}
        \intertext{ with $u \sim N(0, A\Lambda^{L-1}A^{T})$} 
        \Theta^{1} = \Lambda^{0} = \sigma^{2}_{w} X^{T}X + \sigma^{2}_{b}
\end{gather}
\end{restatable} \noindent
The proof of \ref{th:gntk} is given in \ref{th:gntkproof}.
\begin{sproof}\renewcommand{\qedsymbol}{}
    The NTK for Graph Neural Nets can be reduced to the standard NTK by using a property of the  Kronecker Product.
\end{sproof} \noindent
The proof of \ref{th:gntk} is given in \ref{th:gntkproof}.
$\Lambda^{L-1} + (\dot{\Lambda}^{L-1} \odot \Theta^{L-1})$ is positive definite \citep[A.4]{jacot2020neural} or \citep[Theorem 3.1]{provably}. In the context of the GNTK, the NTK's positive definiteness relies on the characteristics of $A$. This might explain why using the Kipf \& Welling normalized adjacency matrix, with eigenvalues in the range $[0,2]$ \citep{kipf2017semisupervised}, aligns with NTK theory (i.e. loss is going to zero for increases width, weights are close to initialization and the empirical NTK stays close to initialization during training, see \ref{recap} for details and definition of \emph{empirical} NTK), whereas the 0-1 adjacency matrix, usually not positive semidefinite is not consistent with NTK theory. For simulations confirming this fact, see \ref{sims}.
The empirical NTK's positive definiteness during training, aligns with key insights about FCNs. For instance, \cite[Lemma 3.3]{provably} links achieving zero training loss to the positive smallest eigenvalue of the empirical NTK. Additionally, \cite{smalleste} demonstrates that the empirical NTK of wide Neural Networks remains positive definite during training under mild assumptions on the data distribution. 
The smallest eigenvalue of the Kernel is related to Generalization Bounds \citep{finegrained} and memorization capacity \citep{montanari2022interpolation}
The GNTK formula may provide insights into the significance of different graph adjacency matrices for GNNs. Also note that the expression reduces to the standard NTK formula for FCNs (see \ref{ntk}) when replacing A with the Identity matrix.
\subsection{Skip-Concatenate Graph Neural Network}\label{sgnn-def}
Skip-Concatenate Connections \citep{chen2020supervised} consist of concatenating the output of the non-linearity and pre-nonlinearity. The first layer is not affected but starting from the second layer, $W^{h}$ with $h>1$ has dimensions $W^{h} \in \mathbb{R}^{d_{l} \times 2d_{l-1}}$.

\begin{restatable}[Skip-Concatenate Graph Neural Network]{definition}{defsgnn}\label{sgnn-definintion}
\begin{align}
F^{h}(X) &= \biggl( \frac{\sigma_{w}}{\sqrt{2d_{h-1}}} W^{h} G^{h-1}(X) + \sigma_{b} B^{h}\biggr) A^{T} \text{ for }h>1 \\
F^{1}(X) &= \biggl( \frac{\sigma_{w}}{\sqrt{d_{1}}} W^{1} G^{0}(X) + \sigma_{b} B^{1}\biggr) A^{T} \\
\text{with } B^{h}&= b^{h} \otimes \mathbf{1}_{n}^{T} \text{ and } b^{h} \in \mathbb{R}^{d_{h}}\\ 
G^{h}(X) &= \begin{pmatrix}
    \sigma(F^{h}(X)) \\
    F^{h}(X)
\end{pmatrix}  \text{ so  } G^{h}(X) \in \mathbb{R}^{2d_{l} \times n} \\
G^{0} &= X \inR{d_{0} \times n}
\end{align}
The Weights are initialized as $W^{h} \sim N(0,1)$ and $b^{h} \sim N(0, 1)$. \\
The final Network of depth $L$ \\
$F^{L}(X) = \biggl( \frac{\sigma_{w}}{\sqrt{d_{L-1}}} W^{L} G^{L-1}(X) + \sigma_{b} B^{L} \biggr) A^{T} \in \mathbb{R}^{d_{L} \times n}$. \\
\end{restatable}

\subsubsection{Skip-Concatenate Graph Neural Network Gaussian Process}
\begin{restatable}[Skip-Concatenate GNN Gaussian Process]{theorem}{thnsgnnp}\label{th:sgnngp} 
$ $\newline
If all the weight dimensions of the hidden layers $d_{L-1},d_{L-1}, ..., d_{1}$ of a Skip-GNN $F(X)^{L}$ successively go to infinity, then $\vect(F(X)^{L}) \sim $ GP(0, ($A \Lambda^{L-1} A^{T}) \otimes I_{d_{L}}$ ), with \\
\begin{gather}
 \Lambda^{L} = \sigma^{2}_{w} \frac{1}{2} \left( \mathbb{E} \left[ \sigma(u) \sigma(u^{T}) \right] + A\Lambda^{L-1}A^{T}\right) +\sigma_{b}^{2} 
        \intertext{  with $u \sim N(0, A\Lambda^{L-1}A^{T})$} 
    \Lambda^{0} = \frac{\sigma^{2}_{w}}{d_{0}} X^{T}X +  \sigma^{2}_{b}
\end{gather}
\end{restatable} \noindent
The proof of Theorem \ref{th:sgnngp} is given in \ref{th:sgnn-proof}.

\subsubsection{Skip-Concatenate Graph Neural Tangent Kernel}
\begin{restatable}[Skip-Concatenate GNN NTK]{theorem}{thnsgntk}\label{th:sgntk}
    Having weight dimensions of the hidden layers $d_{L-1},d_{L-1}, ..., d_{1}$  of a Skip-GNN $F(X)^{L}$, approaching infinity, the NTK is $\Theta^{L} \otimes I_{d_{l}}$
\begin{gather}
\Theta^{L} =  A \left( \Lambda^{L-1} +  \left( \dot{\Lambda}^{L-1}  \odot \Theta^{L-1} \right) \right) A^{T}  \\
  \Lambda^{L} = \sigma^{2}_{w} \frac{1}{2} \left( \mathbb{E} \left[ \sigma(u) \sigma(u^{T}) \right] + A\Lambda^{L-1}A^{T}\right) +\sigma_{b}^{2}  \\
\dot{\Lambda}^{L} = \sigma^{2}_{w}\frac{1}{2} \left( \mathbb{E} \left[ \dot{\sigma}(u) \dot{\sigma}(u^{T}) \right] + 1\right) +\sigma_{b}^{2}
\intertext{ with $u \sim N(0, A\Lambda^{L-1}A^{T})$} 
\Theta^{1} = \Lambda^{0} = \sigma^{2}_{w} X^{T}X +  \sigma^{2}_{b}
\end{gather}
\end{restatable} \noindent
The proof of Theorem \ref{th:sgntk} is given in \ref{th:sgnnntk-proof}.
\subsection{Graph Attention Neural Network}

In this section we recap the Graph Attention Model (GAT) Model from \cite{veličković2018graph}.
We will start by defining a single Attention Layer.
We have $X \inR{d_{0} \times n}$, $W \inR{d_{1} \times d_{0}}$ and $A \inR{n \times n}$.
A standard GNN Layer without bias is $GCN(A, X) = \sigma(WXA^{T})$.
A GAT Layer is $GAT(A, X) = \sigma_{2}(WX \alpha(c, A, WX)))$. The attention matrix commonly used for the GAT is the standard 0-1 Attention Matrix $A_{0\text{-}1}$ with added self loops, so $A=A_{0\text{-}1} + I$. We compute $\alpha(c, A, WX)$ as follows:
\begin{gather}
    H := XW \inR{n \times d_{out}} \\
    M(H, c)_{ij} =  \langle c, \concat(H_{\cdot i}, H_{\cdot j}) \rangle =c_{1}^{T}H_{\cdot i} + c_{2}^{T}H_{\cdot j} 
    \text{ with $c_{1} := c_{[1, ..., d_{out}]}$ and $c_{2} := c_{[d_{out}, ..., 2d_{out}]}$} \\
    \alpha(c, A, G)_{ij} = \frac{\exp \left(\sigma_{1} \left( M(H,c)_{ij} \right)\right)}{\sum_{k \in Neighbors(i) \mathop{\cup} \{i\}} \exp( \sigma_{1} \left( M(H,c)_{ik} \right)}
\end{gather}
The final step consists of computing the row wise $\softm$ only over neighboring nodes of node $i$ and node $i$ itself. The attention weight vector $c \inR{2d_{0}}$ is a learnable parameter. $\sigma_{1}$ is an elementwise nonlinearity like the LeakyRelu.
Multiple GAT Layers' outputs can be stacked vertically for multiheaded attention, either by concatenation or averaging over different heads ($\frac{1}{H}\sum_{h}GAT_{h}(A,X)$).
The GAT model we are using to derive the GP and NTK is a simplified version. We replace one of the $W$s by a copy $\tilde{W}$, simplifying the layer to $GAT(A, X) = \sigma_{2}(WX \alpha(c, A, \tilde{W}X))$. This results in $M(H,c)_{ij} = c_{1}^{T}(\tilde{W}X)_{\cdot i} + c_{2}^{T}(\tilde{W}X)_{\cdot j} = \tilde{c}_{1}^{T}X_{\cdot i} + \tilde{c}_{2}^{T}X_{\cdot j}
$ so we can just disregard $\tilde{W}$ altogether
Disregarding $\tilde{W}$ allows us to apply results from \cite{yang1} and \cite{pmlr-v119-hron20a}. Our simplified GAT model (denoted as GAT*) exhibits no apparent performance loss (see Table \ref{tab:regression} and \ref{tab:classification}).
We replace $\softm$ with an element-wise nonlinearity since there's no closed-form expression for $\mathbb{E}\left[\softm(u)\softm( u^{T} ) \right] $with $u \sim N(0, \Lambda)$.
For the recursive definition of GAT*, we omit the bias term, intending to include it in the final infinite width limit formula without proof (see Corollary \ref{GAT-NTK123}).

\begin{restatable}[GAT*]{definition}{defgat}\label{defi:gat}
We have two elementwise nonlinearities $\sigma_{1}$ and $\sigma_{1}$ which are polynomially bounded, (i.e. $\sigma(x) = c + m|x| \text{ for some } c,m \in \mathbb{R}_{+} $) and $c, c_{1}, c_{2}$ are defined as above, i.e $c^{T} =\concat(c_{1}^{T}, c_{2}^{T})$.
\begin{align}
G^{0} &= X \text{ with $X \inR{ d_{0} \times n}$} \\
L^{l,h}_{i,j} &= A_{ij} \frac{\sigma_{c}}{\sqrt{2{d_{l-1}}}}(c_{1}^{l,h^{T}}G^{l-1}_{\cdot i} + c_{2}^{l,h^{T}}G^{l-1}_{\cdot j}) 
\intertext{ with $c_{1}^{l,h}, c_{2}^{l,h} \inR{d_{l-1}}$} 
F^{l,h} &= G^{l-1} \sigma_{1}(L^{l,h}) \\
F^{l} &= \frac{\sigma_{w}}{\sqrt{H d_{l-1}}} W^{l} \left[ F^{l,1}, F^{l,2}, \cdots, F^{l,H} \right]^{T} \intertext{ with $W \inR{d_{l} \times Hd_{l-1}}$. $F^{l}$ can also be written as}
F^{l}&= \frac{\sigma_{w}}{\sqrt{Hd_{l-1}}}\sum_{h}^{H}  W^{l,h}   F^{l,h} \\
G^{l} &= \sigma_{2}(F^{l}) \\
\intertext{and $W^{l},c^{l,h} \sim N(0,1)$.}
\intertext{The final Network of depth $l$ and heads $H$ is}
F^{l}&=\frac{\sigma_{w}}{\sqrt{Hd_{l-1}}} W^{l} \left[F^{l,1}, F^{l,2}, \cdots, F^{l,H} \right]^{T}
\end{align}
\end{restatable}   \noindent

\subsubsection{Gaussian Process}
\begin{restatable}[GAT* GP]{theorem}{thmgattt}\label{thmgattt}\ \\
If min$\{H,d_{l-1}\}\to \infty$ then $F^{l+1}$ (defined as above) is a Gaussian Process with $\vect(F^{l+1}) \sim \text{GP} \left(0,  \Lambda^{l} \otimes I_{d_{l+1}} \right)$ 
\begin{align}
    \Lambda^{l} &:= \Exp{\sigma_{2}(u)\sigma_{2}(u^{T})} 
        \text{ with u $\sim \text{GP}\left(0, \bm\left(\sigma_{w}^{2} \Lambda^{l-1}, \psi(\Lambda^{l-1}) \right)  \right) $}  \\
   \psi(\Lambda^{l}) &:= \Exp{\sigma_{1}(v)\sigma_{1}(v^{T})} 
   \text{ with  $v \sim \text{GP}\left(0, \sigma_{c}^{2} \gamma_{A}(\Omega) \right) $}
    \intertext{and $ \gamma_{A}(\Omega) := J_{A}
    \begin{psmallmatrix}
    \Omega & \Omega \\
    \Omega & \Omega \\
    \end{psmallmatrix}   J_{A}^{T}    \text{ and $J_{A} = \diag(\vect(A))  \concat \left( (\mathbf{1}_{n} \otimes I_{n}) , (I_{n} \otimes \mathbf{1}_{n}) \right)  $}    $}  
      \Lambda^{0} &= \frac{1}{d_{1}}  X^{T}X 
\end{align}
\end{restatable} \noindent
The proof of Theorem \ref{thmgattt} is given in \ref{gatgpproof}.

\subsubsection{Neural Tangent Kernel}
\begin{restatable}[GAT* NTK]{theorem}{thmgatt}\label{thm:gat2} 
The NTK (of a GAT* of depth $L$) is $\Theta^{L} \otimes I_{d_{l}}$
{\small
\begin{align}
    \Lambda^{l} &:= \Exp{\sigma_{2}(v)\sigma_{2}(v^{T})}), \ \dot{\Lambda}^{l} := \Exp{\dot{\sigma}_{2}(v)\dot{\sigma}_{2}(v^{T})}
    \text{ with  $v \sim \text{GP}\left(0,   \Lambda^{l-1} \right) $}\\
    \dot{\psi}(\Lambda^{l}) &:= \Exp{\dot{\sigma}_{1}(v)\dot{\sigma}_{1}(v^{T})}
    \text{ with  $v \sim \text{GP}\left(0,  \sigma_{c}^{2}\gamma_{A}(\Omega) \right) $}\\
\Theta^{1}&=\Lambda^{0} = \sigma_{w}^{2} \frac{1}{d_{1}} X^{T}X \\
\Theta^{L} &= 
\bm\left[\sigma_{w}^{2} \Lambda^{L-1}, \psi(\Lambda^{L-1})  \right] \\
&+ \bm \left[\sigma_{w}^{2} \Lambda^{L-1} ,   \sigma_{c}^{2} \left(\gamma_{A}(\Lambda^{L-1})   \odot  \dot{\psi}(\Lambda^{L-1})  \right) \right] \\
&+ \bm \left[ \sigma_{w}^{2} \Lambda^{L-1}, \sigma^{2}_{c}\left( \gamma_{A}(\Theta^{L-1} \odot \dot{\Lambda}^{L-1})   \odot  \dot{\psi}(\Lambda^{L-1}) \right) \right] \\
&+ \bm \left[ \sigma_{w}^{2} \left(
            \Theta^{L-1} \odot \dot{\Lambda}^{L-1}\right), \psi(\Lambda^{L-1}) \right]
\end{align}
}%
\end{restatable} \noindent
The proof of Theorem \ref{thm:gat2} is given in \ref{gatntkproof}.

\begin{corollary}
The NTK (of a GAT* of depth $L$ with $\sigma_{1}$  being the Identity Function ) is $\Theta^{L} \otimes I_{d_{l}}$, with
\begin{align}
    \Lambda^{l} &:= \Exp{\sigma_{2}(v)\sigma_{2}(v^{T})}), \ \dot{\Lambda}^{l} := \Exp{\dot{\sigma}_{2}(v)\dot{\sigma}_{2}(v^{T})}
    \text{ with  $v \sim \text{GP}\left(0,   \Lambda^{l-1} \right) $}\\
\Theta^{1}&=\Lambda^{0} = \sigma_{w}^{2} \frac{1}{d_{1}} X^{T}X \\
\Theta^{L} &= 
2 \cdot \bm\left[\sigma_{w}^{2} \Lambda^{L-1}, \sigma_{c}^{2}\gamma_{A}(\Lambda^{L-1})  \right] \\
&+ \bm \left[ \sigma_{w}^{2} \Lambda^{L-1}, \sigma^{2}_{c} \gamma_{A} \left( \Theta^{L-1} \odot \dot{\Lambda}^{L-1} \right) \right] \\
&+ \bm \left[ \sigma_{w}^{2} \left(
            \Theta^{L-1} \odot \dot{\Lambda}^{L-1}\right), \sigma_{c}^{2}\gamma_{A}(\Lambda^{L-1}) \right]
\end{align}
\end{corollary}

\begin{corollary}[GAT* NTK with bias (without proof)]\label{GAT-NTK123}\ \\
    To include a bias term for the GAT NTK one can replace every  occurrence of $\sigma_{w}^{2} \Lambda^{l-1}$ by $\sigma_{w}^{2} \Lambda^{l-1} + \sigma_{b}^{2}$.
\end{corollary}

\subsection{Infinite Width Graph Neural Networks for Graph Regression/ Classification} \label{ind}

In the setting of inductive learning, the task is to train a GNN on multiple graphs, e.g.: \\ ${((A_{1}, X_{1}, y_{1}), (A_{2}, X_{2}, y_{2}), ... ,(A_{n}, X_{n}, y_{n})}$. Where each triple consists of a graph, its features and a label (like a class assignment for classification).
A two layer GNN with a final sum pooling layer, can be written down as: $F^{3}(A_{i},X_{i})_{h} =  \sum_{h} F^{2}(A_{i},X_{i})_{\cdot h}$. So $F^{3}(A_{i},X_{i}) \inR{d_{3}}$ is a vector and $F^{2}(A_{i},X_{i}) \inR{d_{2} \times d_{3}}$ is defined like the GNN for Node Regression/ Classification.
The infinite GP in that case will be $\vect(F^{3}) \sim GP(0, \Sigma) $ with $ \Sigma_{ij} =  \sum_{st} \left(A_{i} \Lambda^{1}A_{j}^{T} \right)_{st}$.
\begin{align*}
    \Lambda^{1} &= \sigma^{2}_{w}\mathbb{E} \left[ \sigma(u) \sigma(u^{T}) \right] +\sigma_{b}^{2} \\
        \intertext{with $u \sim N(0, A_{i}\Lambda^{0}A_{j}^{T})$}
    \Lambda^{0} &= \frac{\sigma^{2}_{w}}{d_{0}} X_{i}^{T}X_{j} + \sigma^{2}_{b}
\end{align*}
The GNTK extends similarly to the setting of Graph Induction. The NTK is $\sum_{st} \Theta^{L}(i,j)_{st}$ with
\begin{gather*}
\Theta^{L}(i,j) = \sum_{st}\left( A_{i} \left( \Lambda^{L-1} +  \left( \dot{\Lambda}^{L-1}  \odot \Theta^{L-1} \right) 
\right) A_{j}^{T} \right)_{st} \text{ with } \\
    \Lambda^{L} = \sigma^{2}_{w}\mathbb{E} \left[ \sigma(u) \sigma(u^{T}) \right] +\sigma_{b}^{2}  \\
        \dot{\Lambda}^{L} = \sigma^{2}_{w}\mathbb{E} \left[ \dot{\sigma}(u) \dot{\sigma}(u^{T}) \right] +\sigma_{b}^{2}
        \intertext{ with $u \sim N(0, A_{i}\Lambda^{L-1}A_{j}^{T})$} 
        \Theta^{1} = \Lambda^{0} = \sigma^{2}_{w} X_{i}^{T}X_{j} + \sigma^{2}_{b}
\end{gather*}
$\Theta^{L}(i,j)$ is just the NTK for Node Classification/ Regression for different data samples $(A_{i},X_{i})$ and $(A_{j},X_{j})$.
The proof can be conducted using the same approach as in \ref{gnn123}.

\clearpage
\section{Experiments}
All models were evaluated on a variety of transductive Node Classification/ Regression Tasks.
We are going to first summarize the datasets, proceed to the simulations showing how GNNs are behaving for increasing width and finally specify the models in detail and present the performance.

\subsection{Simulations} \label{sims}
In this section we empirically confirm the NTK findings from \ref{recap}. 1) consists of showing the for increased width the loss converges to zero, 2) consists of showing that the weight barely change for increased width and 3) consists of showing that the NTK is close to its initialization during training.
We use the KarateClub dataset \citep{karate} with 34 nodes, of which we used 27 for training with four classes. We now show for FCNs, GNNs, Skip-GNNs and the GAT that 1), 2) and 3) are valid (see Fig. \ref{fig:ntksim}). They were initialized randomly as defined in \ref{fcn-ef}, \ref{defi:gnn}, \ref{sgnn-definintion} and \ref{defi:gat}.
For the GNN and Skip-GNN we use the Kipf \& Welling normalized adjacency matrix and for the GAT the 0-1 adjacency with added self loops.
Finally, using a GNN with the 0-1 adjacency with self loops is not consistent with NTK theory. Figure \ref{fig:ntksimfails} shows cleary that with increased width neither the weights are changing less nor does the empirical NTK stay close to its initialization. For a detailed discussion on why, see \ref{th:gntk}.

\begin{table}[htb!]
\centering
    \begin{tabular}{|c|c|c|c|} \hline
         &  Wiki &  Cora& Citeseer\\ \hline
         GCN-8&  0.09, 0.025&  0,14, 0.22& 0.17, 0.16\\ \hline
         GCN-32&  0.07, 0.32&  0.14, 0.25& 0.18, 0.2\\ \hline
         GCN-128&  0.07, 0.45&  0.15, 0.13& 0.17, 0.16\\ \hline
         GCN-512&  0.08, 0.49&  0,16, 0.1& 0.17, 0.16\\ \hline
    \end{tabular}
    \caption{Overfitting with Gradient Descent, comma-seperated loss \& accuracy}
    \label{tab:overf1}
    \centering
    \begin{tabular}{|c|c|c|c|} \hline
         &  Wiki &  Cora& Citeseer\\ \hline
         GCN-8&  0.1, 0.15&  0, 0.75& 0.017, 0.57\\ \hline
         GCN-32&  0.06, 0.08&  0, 0.73& 0, 0.6\\ \hline
         GCN-128&  0.07, 0.45&  0.15, 0.13& 0.17, 0.16\\ \hline
         GCN-512&  0.01, 0.73&  0, 0.76& 0, 0.625\\ \hline
    \end{tabular}
    \caption{Overfitting with Adam Optimizer, comma-seperated loss \& accuracy}
    \label{tab:overf2}
\end{table} \noindent
Table \ref{tab:overf1} displays the training loss and accuracy comma-seperated for a two layer GNN with increasing width for different datasets. In table \ref{tab:overf1} the GNN was trained with Gradient Descent for with a learning rate of 0.001 and in table \ref{tab:overf2} Adam Optimizer with the same learning rate. For a FCN we would expect a very wide Neural Net to achieve zero training loss and \emph{not} overfit \citep{arora2019exact}. For GNNs we can observe that Gradient Descent fails completely, although the loss in each case is close to zero. For the adam optimizer
increased width, improves accuracy with an outlier when we set the width to 128. The loss is already near zero for setting the width to 8.  GNNs seem to be very sensitive to the optimizer, step size and width more than would be expected for FCNs.

\begin{figure*}
 \begin{subfigure}{0.33\textwidth}
                \includegraphics[width=\linewidth]{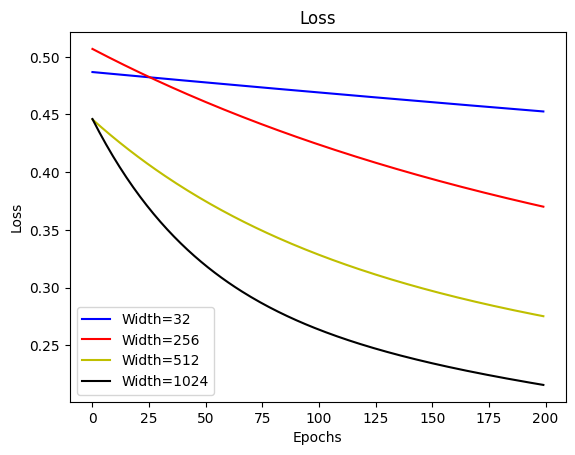}
        \end{subfigure}%
         \begin{subfigure}{0.33\textwidth}
                \includegraphics[width=\linewidth]{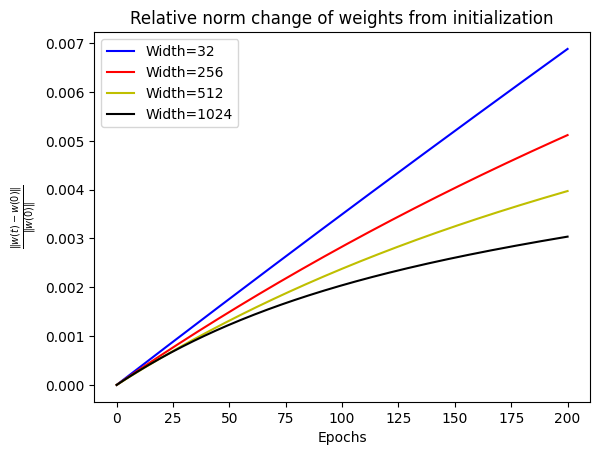}
        \end{subfigure}%
         \begin{subfigure}{0.335\textwidth}
                \includegraphics[width=\linewidth]{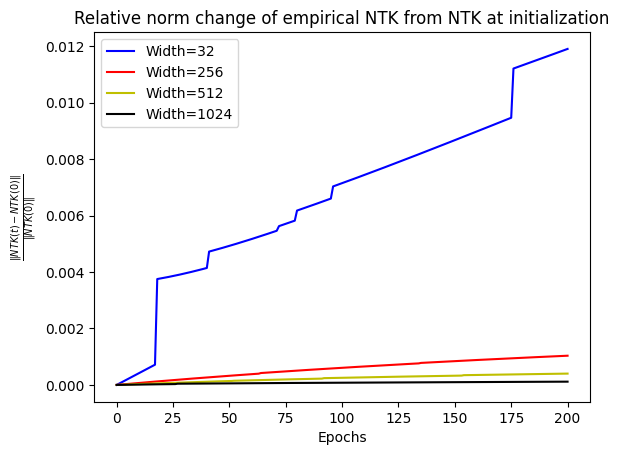}
        \end{subfigure}%

     \begin{subfigure}{0.33\textwidth}
                \includegraphics[width=\linewidth]{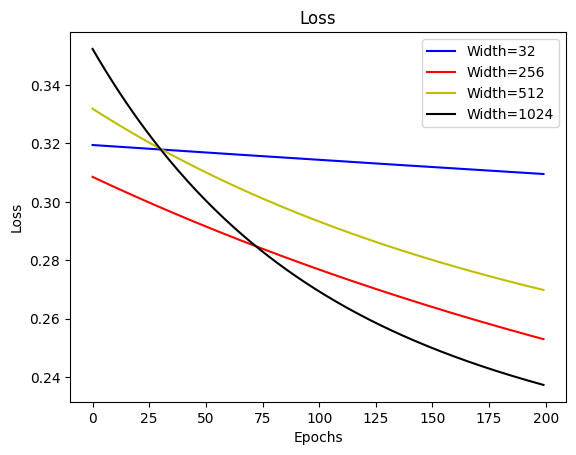}
        \end{subfigure}%
         \begin{subfigure}{0.33\textwidth}
                \includegraphics[width=\linewidth]{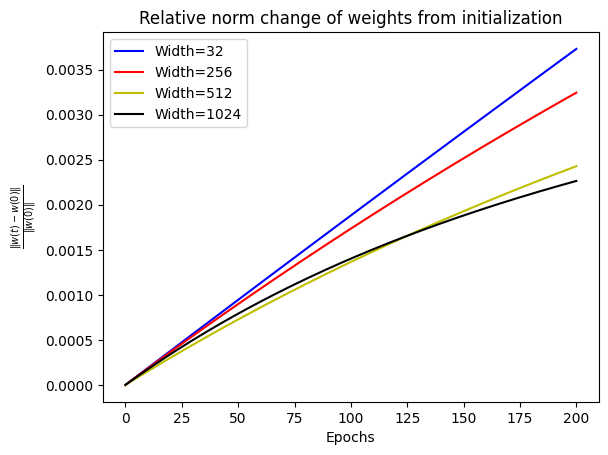}
        \end{subfigure}%
         \begin{subfigure}{0.335\textwidth}
                \includegraphics[width=\linewidth]{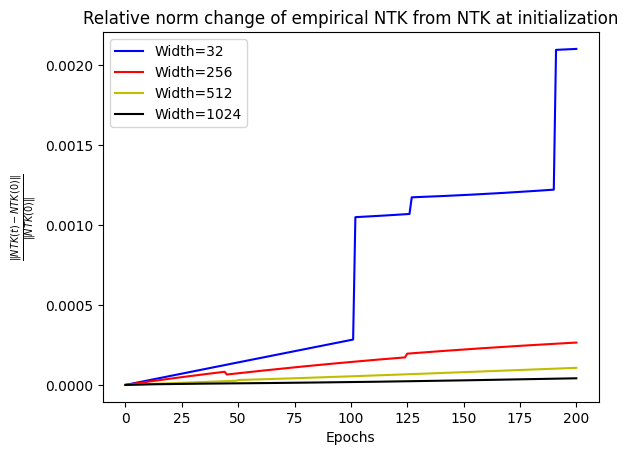}
        \end{subfigure}%
        
             \begin{subfigure}{0.33\textwidth}
                \includegraphics[width=\linewidth]{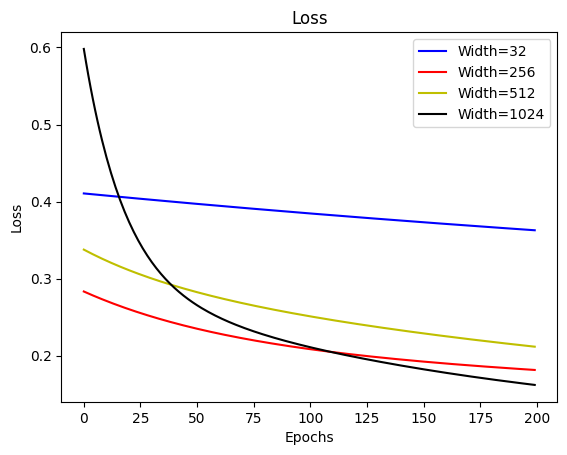}
        \end{subfigure}%
         \begin{subfigure}{0.33\textwidth}
                \includegraphics[width=\linewidth]{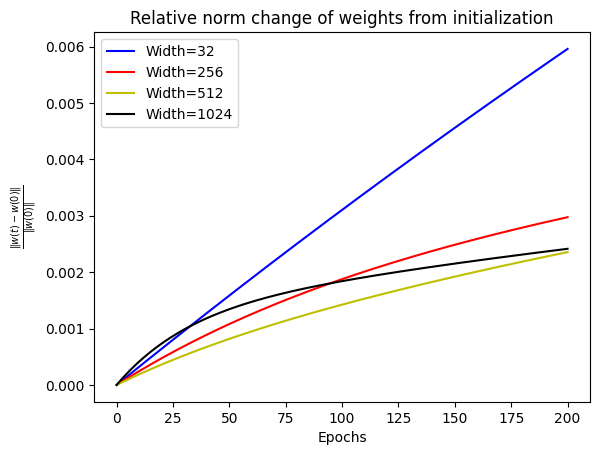}
        \end{subfigure}%
         \begin{subfigure}{0.335\textwidth}
                \includegraphics[width=\linewidth]{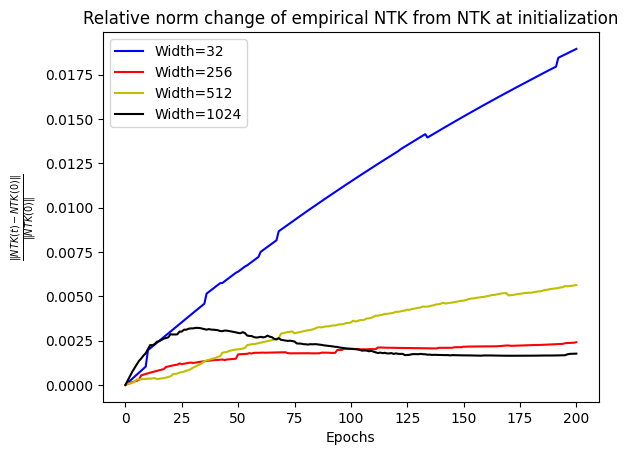}
        \end{subfigure}%

\begin{subfigure}{0.33\textwidth}
                \includegraphics[width=\linewidth]{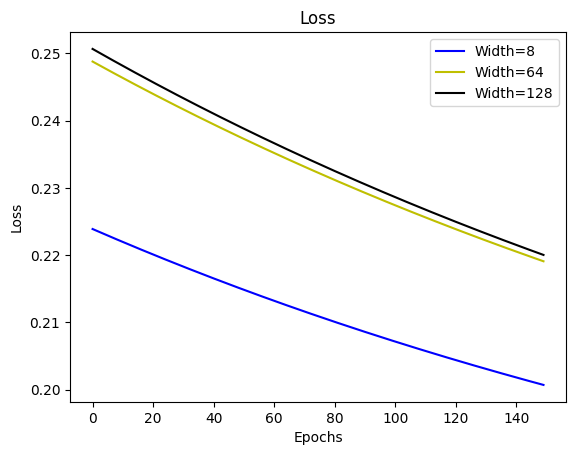}
        \end{subfigure}%
         \begin{subfigure}{0.33\textwidth}
                \includegraphics[width=\linewidth]{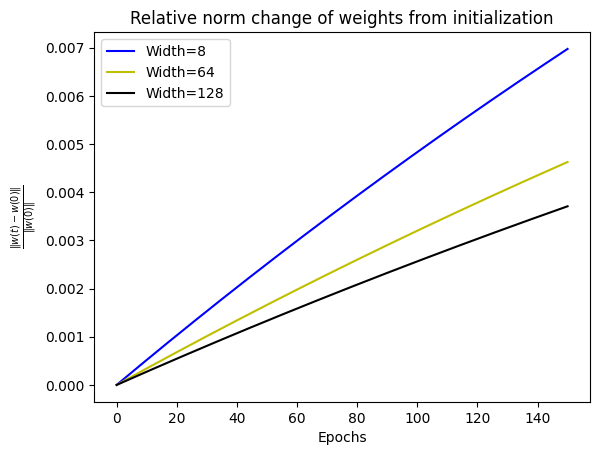}
        \end{subfigure}%
         \begin{subfigure}{0.335\textwidth}
                \includegraphics[width=\linewidth]{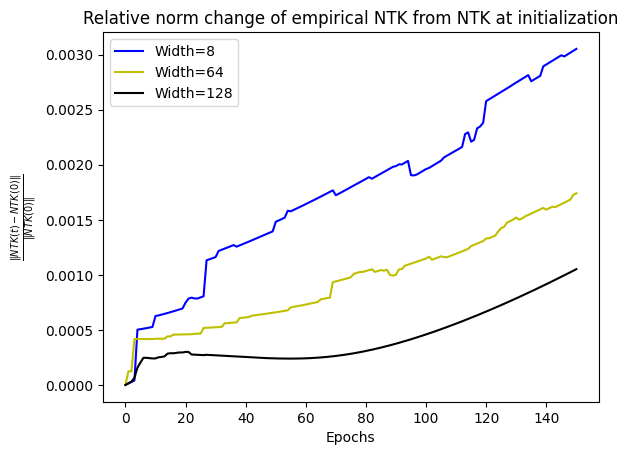}
        \end{subfigure}%
\caption{First row corresponds to a two layer FCN, second row to a two layer GNN with the Kipf \& Normalized Adjacency Matrix, third row is the Skip-GNN with the same Adjacency Matrix and fourth row a GAT with the number of Heads equal to the Width and the Attention Heads are summed over. All models are trained with learning rate of 0.001 and Gradient Descent.
For each row the left figures shows training loss for different widths, the middle figure shows weight change (difference to weights at initialization) during training and the right figures show the difference of the empirical NTK compared to its initialization during training. These simulations confirm that for increased width during training we can observe 1) the training loss approaches zero 2) the weights of the GNN stay close to its initialization and 3) the empirical  NTK stays close to its initialization during training.}
\label{fig:ntksim}
\end{figure*}

\begin{figure}[!htb]
 \begin{subfigure}{0.33\textwidth}
                \includegraphics[width=\linewidth]{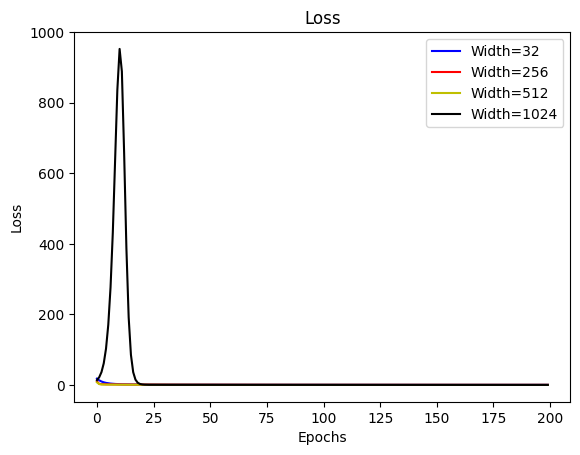}
        \end{subfigure}%
         \begin{subfigure}{0.33\textwidth}
                \includegraphics[width=\linewidth]{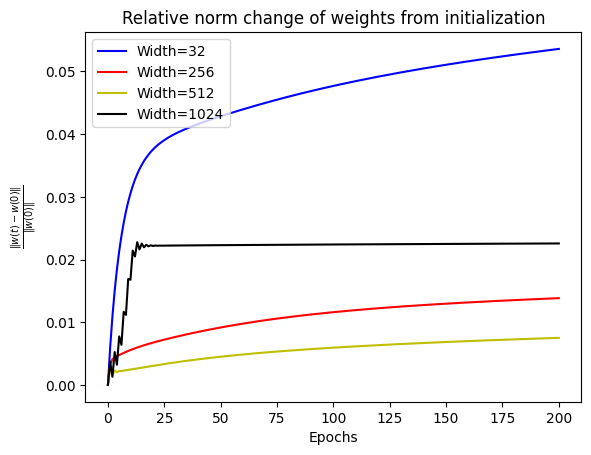}
        \end{subfigure}%
         \begin{subfigure}{0.335\textwidth}
                \includegraphics[width=\linewidth]{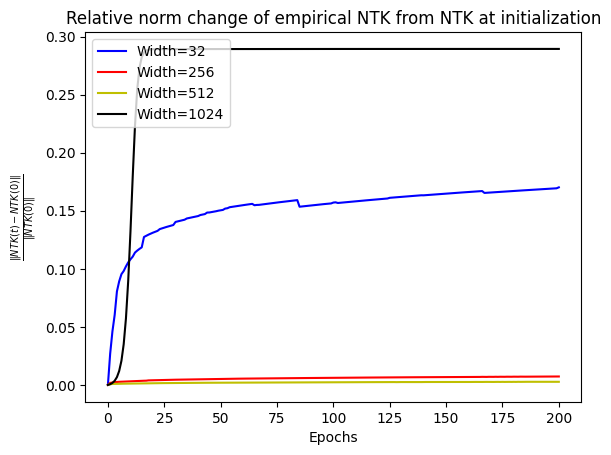}
        \end{subfigure}%
\caption{A GNN trained with the 0-1 Adjacency with self loops for a two layer GNN with a learning rate of 0.001 is not consistent with 1), 2) and 3).}
\label{fig:ntksimfails}
\end{figure}

\subsection{Datasets, Models \& Evaluation}
\subsubsection*{Datasets}
The following section contains a short description of each dataset:
\begin{enumerate}
    \item Classification: The citation network dataset ``Cora'', ``CiteSeer'' and ``PubMed'' from \cite{yang2016revisiting} and ``Wiki'' from \cite{wiki} are Node Classification Tasks. 
    The performance measure used is Accuracy.
    \item Regression:   are Node Regressions Tasks which are part of the Wikipedia networks Dataset \citep{multi}. Nodes represent web pages and edges represent hyperlinks between them. Node features represent several informative nouns in the Wikipedia pages. The task is to predict the average daily traffic of the web page. The performance measure used is the $R^{2}$-Score.
\end{enumerate}

\begin{table}[htb!]

  \label{tab:dataset}
  \begin{center}
  \begin{tabular}{ccccccc}
    \hline
    Dataset & Task  & Nodes & Edges & Features & Classes &Train/Val/Test Ratio \\
    \hline
    Cora  & Classification & 2,708  & 10,556& 1,433  &7 & 0.05/0.18/0.37 \\
    Citeseer  & Classification & 3,327  & 9,104& 3,703 &6 &  0.04/0.15/0.30 \\
    PubMed  & Classification & 19,717  & 88,648& 500  &3&  0.003/0.025/0.051 \\
    Wiki  & Classification & 2,405  & 17,981& 4,973& 17&  0.60/0.20/0.20 \\
    Chameleon & Regression & 2,277  & 62,742  & 3,132 &Regression & 0.48/0.32/0.20 \\
    Squirrel & Regression & 5,201  & 396,706  & 3,148 &Regression & 0.48/0.32/0.20 \\
    Crocodile & Regression & 11,631  & 341,546  & 13,183 & Regression& 0.48/0.32/0.20 \\
    \hline
  \end{tabular}

  \end{center}
      \caption{Dataset summary}
\end{table} \noindent
The datasets splits are exactly modelled after \citep{kipf2017semisupervised, veličković2018graph}. Note that the train/val/test split is fixed and not randomly sampled for each experiment. Not all graphs have labels for all nodes (e.g. Cora). 
\clearpage
\subsection*{Models}
Now we summarize the hyperparameters for the Models: 

\subsubsection*{Neural Nets}
The FCN, GNN, S-GNN, GAT and GAT* are all Neural Networks.
All but the S-GNN have two layers, as added depth did not increase the performance of the models.
The Skip-GNN has three layers, as the Skip-Connections only take effect after a minimum of three layers. The FCN is a Neural Net just trained on the Node features, i.e. a GNN with $I$ as the Adjacency Matrix. 
The models were trained for 300 epochs without early stopping. The features were normalized before training. All models used the Adam Optimizer with weight decay of 0.005. The GNN, S-GNN used a learning rate of 0.01 whereas the GAT and GAT* used a learning of 0.005.
Dropout of 0.5 was applied to the output of every layer during training. In the case of GAT and GAT* dropout of 0.5 was also applied to the Attention Adjacency Matrix. The Neural Nets were trained with Cross entropy Loss for Classification and Mean-Squared Error for Regression. The hyperparameters are borrowed from \cite{kipf2017semisupervised, veličković2018graph} where they were determined on the validation sets. Contrary to the mentioned papers, we did no use early stopping.
Regression was done similar as above but without feature normalization and for FCN, GNN, Skip-GCN with a learning rate of 0.1 without weight decay but with batch normalization after each layer except the last layer and for GAT and GAT* with a Learning Rate of 0.01 and no Batch Normalization. The FCN for Regression had Sigmoid nonlinearity instead of Relu, as it improved the performance from a negative $R^{2}$-Score to a positive one.

\subsubsection*{Kernels and Gaussian Processes}
The Kernels and GP were trained without feature normalization.
We used the validation set to perform a grid search for values between 0.001 and 10 to find the best regularization parameters. Regularization for GP (sometimes referred to as Noise Parameter) results in performing Kernel Ridge Regression, which results in adding $\lambda I$ to the Kernel/GP Covariance Matrix before inverting them (see e.g. \cite{rasm} for connections between GP and Kernels). 
 $\sigma_{w}^{2}$ and $\sigma_{c}^{2}$ are set to one and $\sigma_{b}^{2} $ is set to zeros for all Classification tasks. Regression was done with $\sigma_{b}^{2}= 0.1 $ similar to the setup in \cite{niu2023}. We conduct additional experiments (see \ref{roleofb}) to investigate the role of $\sigma_{b}^{2}$ as it seems to have a crucial role in the performance for the Regression tasks.
 In the case of GAT*-GP and GAT*-NTK, the nonlinearity for the Attention Matrix (e.g $\sigma_{1}$) was chosen to be the Identity Function. $\sigma_{2}$ was the LeakyRelu with a slope of 0.2 for the negative slope. In \ref{tab:gateval} we compared different nonlinearities for the GAT*GP and GAT*NTK. GAT*GP/GAT*NTK-LINEAR have no nonlinearity. GAT*GP/GAT*NTK-2 have two nonlinearities $\sigma_{2}$ was the LeakyRelu as above and $\sigma_{1}$ is the Error-function\footnote{\url{https://en.wikipedia.org/wiki/Error_function}}. We can see that adding two nonlinearities can improve the accuracy by 1\%-2\%. GAT*GP and GAT*NTK in general performed very bad for Regression tasks with a negative $R^{2}$-Squared Score.

\subsection*{Evaluation}
Due to the large memory and time requirements for GAT*GP and GAT*NTK we ``sparsified'' all graphs denoted with \# by removing 90\% of their edges  using a technique called Effective Resistance \citep{spielmanS11}. For details, see \ref{ef}).
The experiments clearly show that the NTK and GP are competitive with the Neural Net Counterpart even when the Neural Network is not a Vanilla Neural Net trained using the MSE Loss with SGD and extensions like Dropout turned off. In the case of FCNs it was shown that their NTK counterpart had a significant performance gap of between 6\% - 10\% when compared to their vanilla counterparts \citep{arora2019exact}, let alone having things like Dropout, Batch Normalization and so on.  Batch Normalization can be incorporated into the GP/NTK framework \cite{yang2}. Most models seem to be performing very close within each other performance without any one standing out.
The GNTK appears to be a strong choice, slightly outperforming the GNNGP in Regression.
In all cases a GP/NTK is able to achieve matching or even better performance than their Neural Net Counterparts.
The GAT Variants all seem to do very poorly for Regression tasks, with even negative $R^{2}$-Scores for the GP \& NTK counterparts. Overall using the Kernels and GP is preferable because it does not require tuning hyperparameters and the GNN are very sensitive for hyperparameters (e.g. number of hidden nodes, see \ref{tab:overf2}).
\begin{figure}[!htb]
\hspace{-0.5cm}\begin{minipage}[c]{0.5\textwidth}
\centering
\begin{tabular}{|c|c|c|c|c|} \hline 
         &  Cora  &  Citeseer &  Pubmed &  Wiki   \\ \hline 
         FCN&  0.61&  0.59&  0.73&  0.72\\ \hline 
         GNN     &  0.81&  0.71&  0.79& 0.66 \\ \hline 
         Skip-GNN   &  0.81&  0.71&  0.79&  0.75\\ \hline 
         GAT&  0.82&  0.71  &  0.77&  0.66\\ \hline 
         GAT*    &  0.82&  0.70&  0.77&  0.65\\ \hline 
         NNGP    &  0.60&  0.62&  0.73&  \textbf{0.81}\\ \hline 
         NTK     &  0.58&  0.62&  0.72&  \textbf{0.81}\\ \hline 
         GNNGP&  \textbf{0.83}&  0.71&  \textbf{0.80}&  0.78\\ \hline 
         GNTK    &  \textbf{0.83}&  \textbf{0.72}&  0.79&  0.79\\ \hline 
         SGNNGP&  \textbf{0.83}&  0.71&  \textbf{0.80}&  0.78\\ \hline 
         SGNTK    &  0.80&  \textbf{0.72}&  0.79&  0.78\\ \hline
        GAT*GP&  0.79&  0.71& 0.73\# & 0.78\\ \hline 
         GAT*NTK    &  0.79&  0.71& 0.73\# &  0.77\\ \hline
    \end{tabular}
    \captionof{table}{Classification}
    \label{tab:regression}
\end{minipage}
\begin{minipage}[c]{0.5\textwidth}
\hspace{0.5cm}\begin{tabular}{|c|c|c|c|} \hline 
         &  Chameleon  &  Squirrel &  Crocodile   \\ \hline 
         FCN&    0.52&  0.39&  0.75\\ \hline 
         GNN      &  0.48&  0.35& 0.64\\ \hline 
         Skip-GNN     &  0.38&  0.31&  0.58\\ \hline 
         GAT  &  0.43&  0.30&  0.66\\ \hline 
         GAT*    &  0.54&  0.30&  0.65\\ \hline 
         NNGP     &  0.63&  0.45&  0.79\\ \hline 
         NTK       &  0.67&  0.48&  \textbf{0.80}\\ \hline 
         GNNGP&    0.64&  0.48&  0.78\\ \hline 
         GNTK      &  \textbf{0.68}&  \textbf{0.51}&  0.79\\ \hline 
         SGNNGP&    0.59&  0.43&  0.64\\ \hline 
         SGNTK      &  0.57&  0.46&  0.55\\ \hline
        GAT*GP&    -12.14& -8.6\#& 16.94\#\\ \hline 
         GAT*NTK      &  -8.75& -8.2\#&  -17.69\#\\ \hline
    \end{tabular}
    \captionof{table}{Regression}
    \label{tab:classification}
\end{minipage}
\end{figure}

\begin{table}[!htb]
    \centering
    \begin{tabular}{|c|c|c|c|} \hline 
         &  Cora  &  Citeseer  &  Wiki  \\ \hline 
          GAT*GP-LINEAR&  0.79&  0.72&  0.74\\ \hline 
         GAT*NTK-LINEAR    &  0.79&   0.72& 0.74\\ \hline
          GAT*GP&  0.79&  0.71&  0.78\\ \hline 
         GAT*NTK    &  0.79&   0.71& 0.77\\ \hline
        GAT*GP-2&  0.79&  0.71&  0.78\\ \hline 
         GAT*NTK-2    &  0.8&   0.71& 0.80\\ \hline
    \end{tabular}
    \caption{Different GAT models with different nonlinearities.}
    \label{tab:gateval}
\end{table}

\subsection{Code \& Environment}
All Experiments were conducted on an AMD Ryzen 7 3800X 8-Core Processor with hyperthreading with 32GB Ram. Experiments with † were conducted on an Intel(R) Xeon(R) Gold 6148 CPU @ 2.40GHz with 20 Cores without hyperthreading with 362GB of Ram. All Experiments were conducted using the CPU. The measurements for GAT*GP and GAT*NTK Pubmed, Squirrel and Crocodile are the datasets with 90\% of edges removed via Effective Resistance.
The  GAT*GP and GAT*NTK code is highly optimized, making use of the Python Intel MKL Library which has an implementation of parallel Sparse Matrix multiplication whereas Pytorch and Scipy perform Sparse Matrix operations on one core only. 
In depth explanation and code is available at \url{https://github.com/yCobanoglu/infinite-width-gnns}.

\begin{table}[!htb]
\begin{tabular}{|l|l|l|l|l|l|}
\hline
                                                                              & Cora/Citeseer/Wiki                                                              & Pubmed                                                    & Chameleon                                                & Squirrel                                                 & Crocodile                                                \\ \hline
Neural Nets                                                                   & \textless{}1min                                                                 & \textless{}2min                                           & \textless{}2min                                          & \textless{}3min                                          & \textless{}3min                                          \\ \hline
\begin{tabular}[c]{@{}l@{}}NNGP,NTK,\\ GNNGP, GNTK\\ SGNN, SGNTK\end{tabular} & \textless{}2s                                                                   & 2min                                                      & 1.15min                                                  & \textless{}3min                                          & \textless{}3min                                          \\ \hline
\begin{tabular}[c]{@{}l@{}}GAT*GP,\\  GAT*NTK\end{tabular}                    & \begin{tabular}[c]{@{}l@{}}10min, 15min, 1h\\  \textless{}16GB Ram\end{tabular} & \begin{tabular}[c]{@{}l@{}}25h, \\ 60GB Ram†\end{tabular} & \begin{tabular}[c]{@{}l@{}}4h, \\ 32GB Ram†\end{tabular} & \begin{tabular}[c]{@{}l@{}}1h, \\ 32GB Ram†\end{tabular} & \begin{tabular}[c]{@{}l@{}}1h, \\ 32GB Ram†\end{tabular} \\ \hline
\end{tabular}
\end{table}

\subsection{Effective Resistance Spectral Graph Sparsification}\label{ef}
Recall that the GAT*GP and GAT*NTK require calculating the expression:
\begin{gather*}
   \psi(\Lambda^{l}) := \Exp{\sigma_{1}(v)\sigma_{1}(v^{T})} 
    \text{ with  $v \sim \text{GP}\left(0, \sigma_{c}^{2} \gamma_{A}(\Lambda^{l-1}) \right)
      \text{and } \gamma_{A}(\Omega) := J_{A}
    \begin{psmallmatrix}
    \Omega & \Omega \\
    \Omega & \Omega \\
    \end{psmallmatrix}   J_{A}^{T} $}
\end{gather*}
 and $J_{A} = \diag(\vect(A))  \concat \left( (\mathbf{1}_{n} \otimes I_{n}) , (I_{n} \otimes \mathbf{1}_{n}) \right)$ which comes from \ref{thmgattt}. \\
The expression  $ \gamma_{A}(\Omega):= J_{A}
    \begin{psmallmatrix}
    \Lambda^{l} & \Lambda^{l} \\
    \Lambda^{l} & \Lambda^{l}
    \end{psmallmatrix}
    J_{A}^{T}$  will be $\inR{n^{2} \times n^{2}}$ where $A \inR{n \times n}$.  For Pubmed this will be a sparse $400,000 \times 400,000$ matrix. Calculating $\psi(\Lambda^{l}) := \Exp{\sigma_{1}(v)\sigma_{1}(v^{T})} $ will again need multiple operations on the same matrix.
Note that in the case of $\sigma_{1}$ being the Identitiy Function $\psi(\Lambda^{l}) = \gamma_{A}(\Omega)$. To still incorporate  $\sigma_{1}$ in an efficient manner recall the definition of a two layer GAT Model:
\begin{gather*}
    G^{1} = W^{1}X \\
    F^{2} = \sigma_{2}(G^{1} \sigma_{1}(A \odot C)) \text{ with } C_{ij}= c_{1}^{T}G^{1}_{\dot i} + c_{2}^{T}G^{1}_{\dot j}
\end{gather*}
with A being the 0-1 adjacency and with added self loops and $\sigma_{1}(0) = 0$ (.e.g. Relu, Sigmoid), results in $ F^{2} = \sigma_{2}(G^{1} ( A \odot \sigma_{1}(C)))$.
The corresponding Kernel would be
\begin{gather*}
   \psi(\Lambda^{l}) := \gamma_{A} \biggl( \Exp{\sigma_{1}(v)\sigma_{1}(v^{T})} \biggr)
    \text{ with  $v \sim \text{GP}\left(0, \sigma_{c}^{2} \Lambda^{l-1} \right)$}
\end{gather*}
This way we $\Exp{\sigma_{1}(v)\sigma_{1}(v^{T})}$ will now be much more efficient as the resulting matrix will be $\inR{n \times n}$ instead $\inR{n^{2} \times n^{2}}$.
Calculating $\gamma_{A}(\Omega)$ for large graphs/matrices is still difficult and that is why we used Spectral Graph Sparsification method, namely Effective Resistance on PubMed, Squirrel and Crocodile for GAT*GP and GAT*NTK. EffectiveResistence \citep{spielmanS11} 
was successfully applied for GNNs and GATs and one could show that it is possible to remove as much as 90\% of the edges of a graph without sacrifice much performance and improving memory and runtime \citep{fast_graph}. Effective Resistance takes an unweighted 0-1 adjacency matrix and returns a sparsified (i.e. graph with fewer edges) weighted adjacency matrix such that the Laplacians of these adjacency matrices are close. In our use case, we apply Effective Resistance to the 0-1 adjacency matrix and then replace each edge weight with the value one, so it is again a 0-1 adjacency matrix.
This way we could improve speed especially for Pubmed considering it took around 30 hours and 60GB Ram despite the fact that 90\% of the edges had been removed. The following tables demonstrate Effective Resistance on GNNGP and GNTK Models. We run experiments with 50\% and 90 \% of the edges removed. There is a performance penalty for removing edges, but the difference in Cora, Citeseer, and Pubmed between removing 90\% or 50\% is negligible.

\begin{table}[htb!]
    \centering
    \begin{tabular}{|c|c|c|c|c|c|c|c|} \hline 
    &  Cora  &  Citeseer &  Pubmed &  Wiki & Chameleon & Squirrel & Crocodile   \\ \hline 
    GNNGP &  0.83&  0.71&  0.80&  0.78&  0.64&  0.48& 0.78\\ \hline 
         GNTK &  0.83&  0.72&  0.79&  0.79&  0.68&  0.51& 0.79\\ \hline 
         GNNGP 0.5&  0.77&  0.66&  0.75&  0.70&  0.52&  0.4& 0.45\\ \hline 
         GNTK 0.5&  0.76&  0.67&  0.75&  0.71&  0.52&  0.46& 0.15\\ \hline 
         GNNGP 0.9&  0.77&  0.69&  0.77&  0.40&  0.45&  0.28& 0.42\\ \hline 
         GNTK 0.9&  0.78&  0.69&  0.77&  0.40&  0.45&  0.27& 0.38\\ \hline 
    \end{tabular}
    \caption{Performane with Graph Spectral Sparsification known as Effective Resistance}
    \label{tab:effective}
\end{table}

\subsection{Role of bias\ $\sigma_{b}^{2}$  for Regression}\label{roleofb}
Regression experiments with GPs and Kernels we set $\sigma_{b}^{2}=0.1$ and for the GNNs for Regression we added batch normalization, so we could reproduce the results from \cite{niu2023}.
We are going to explore the performance for 1) Having no bias parameter for the GNNs 2) Setting $\sigma_{b}^{2}=0$ for the GP and NTK with the Sigmoid nonlinearity 3) Setting $\sigma_{b}^{2}=0$ for the GP and NTK with the Relu nonlinearity 4) $\sigma_{b}^{2}=0$ for GNNGP and GNTK with the Relu nonlinearity.
For Classification, trying different  $\sigma_{b}^{2}$ or having the models without bias term did not have the same effect as for Regression (i.e. no large decrease in performance).
For Regression, the bias term for the Neural Nets as well as the $\sigma_{b}^{2}$ seems to play a crucial role and by leaving it out the performance worsens in practically every example (including negative $R^{2}$-Scores).

\begin{table}[!htb]
    \centering
    \begin{tabular}{|c|c|c|c|} \hline 
              &  Chameleon&  Squirrel&  Crocodile\\ \hline
        FCN (No Bias) &  0.54&  0.38&  0.75\\ \hline
        GNN (No Bias) &  -0.13&  -1.51&  0.17\\ \hline
        NNGP (Sigmoid) ($\sigma_{b}=0$) &  0.46&  0.40&  0.75\\ \hline
        NTK (Sigmoid) ($\sigma_{b}=0$) &  0.46&   0.40&  0.75\\ \hline
        NNGP (Relu)($\sigma_{b}=0$) &  -1.59&  -7.917&  0.726\\ \hline
        NTK (Relu) ($\sigma_{b}=0$) &  -1.481&   -7.314&  0.725\\ \hline
        GNNGP ($\sigma_{b}=0$) &  0.175&  -0.817&  0.299\\ \hline
        GNN ($\sigma_{b}=0$) &  0.202&  -0.757&  0.34\\ \hline
    \end{tabular}
    \caption{Role of bias\ $\sigma_{b}^{2}$  for Regression}
    \label{tab:roleofb}
\end{table} \noindent

\section{Discussion \& Conclusion}
Developing closed form expression for different GNN architectures is an important step in paving the way to generalize  results from NTK Theory for GNNs. As discussed in \ref{th:gntk}, the spectrum of the NTK plays a crucial role and needs to be investigated further. The GAT models are particularly interesting because the adjacency matrices are learned, adding to the complexity. Nevertheless, NTK Theory has to be considered with care as distinctions between GNNs and FCNs arises. GNNs appear more susceptible to overfitting when overparametrized and employing optimization algorithms like Adam, as opposed to GD/SGD, becomes essential for achieving good results (see \ref{tab:overf1}). \\
\\
We derived new Kernel and Gaussian Processes by using their connection to infinite Width Neural Networks.
We could successfully show that their performance is competitive with recent Graph Neural Networks architectures, makes them a valid alternative for Machine Learning tasks on graph structured data. Gaussian Processes and Kernels are easy to implement and allow for uncertainty estimation. 
Unlike Graph Neural Networks, which are highly sensitive to various hyperparameter settings, GP and NTK counterparts practically require no hyperparameters. This makes them valuable for establishing robust performance benchmarks across different tasks.
We hope for this work to serve as a source of inspiration for the development of innovative Neural Network architectures and Kernel designs.

\appendix

\section{Appendix}
\subsection{Fully-connected Deep Nets}\label{a1}
We copied the definitions of the Graph Neural Network from the main section.

\deffcn*

\subsubsection*{Derivatives}
The derivatives will be necessary for the closed form of the NTK. \\
The derivative of $F(\theta, X)$ with respect $W^{l}$ and $b^{l}$ can be derived using the chain rule.
\begin{align*}
\p{\vect(F(\theta, X))}{\vect(W^{l})} = \p{\vect(F^{L})}{\vect(G^{L-1})} \p{\vect(G^{L-1})}{\vect(F^{L-1)}} \cdots  \p{\vect(F^{l+1})}{\vect(G^{l})}
\p{\vect(G^{l})}{\vect(F^{l})} \p{\vect(F^{l})}{\vect(W^{l})}
\end{align*}
Using 
\begin{align*}
\p{\vect(F^{h})}{\vect(G^{h-1})} &= \frac{\sigma_{w}}{\sqrt{d_{h-1}}}(I_{n} \otimes  W^{h})  \\
\Sigma^{h} := \p{\vect(G^{h})}{\vect(F^{h})} &=  \diag(\vect(\dot{\sigma}(F^{h})) \\
\p{\vect(F^{h})}{\vect(W^{h})} &=   \frac{\sigma_{w}}{\sqrt{d_{h-1}}}(G^{h-1^{T}} \otimes I_{d_{h}}) \\
\p{\vect(F^{h})}{b^{h}} &=  \sigma_{b} (\mathbf{1}_{n} \otimes I_{d_{h}})
\end{align*}
Combining the derivatives with the chain rule the final expression is:
\begin{align*}
\p{\vect(F(X))}{\vect(W^{l})} &=  \frac{\sigma_{w}}{\sqrt{d_{h-1}}}(I_{n} \otimes W^{L}) \Sigma^{L-1}  \cdots \frac{\sigma_{w}}{\sqrt{d_{h-1}}}(I_{n} \otimes W^{l})\Sigma^{l} \frac{\sigma_{w}}{\sqrt{d_{h-1}}}( G^{l-1^{T}} \otimes I_{d_{h}} ) \\
\intertext{and similarly}
\p{\vect(F(X))}{b^{l}} &=  \frac{\sigma_{w}}{\sqrt{d_{h-1}}}(I_{n} \otimes W^{L}) \Sigma^{L-1}  \cdots \frac{\sigma_{w}}{\sqrt{d_{h-1}}}(I_{n} \otimes W^{l})\Sigma^{l} \sigma_{b} (\mathbf{1}_{n} \otimes I_{d_{h}})
\end{align*}

\subsubsection{Gaussian Process (Theorem~\ref{th:nngp})}  \label{th:nngpproof}
\thnngp*
\begin{proof}
The matrix $\mathrm{Cov}(\text{Y})_{IJ} \in \mathbb{R}^{d_{1} \times d_{1}}$ will correspond to the matrix $\mathrm{Cov}(\text{Y}_{\cdot i}, \text{Y}_{\cdot j})$. And similarly $\mathbb{E}(Y)_{I} \in \mathbb{R}^{d_{1}}$  will be $\mathbb{E}(Y_{ \cdot i})$. We will first derive the $\mathbb{E}(x)$ and $\mathrm{Cov}(x, x')$ with respect to a single, respectively two data samples and then derive the Expectation and Covariance for the whole dataset.
Base Case for single datasample, respectively two data samples:
\begin{align*}
\mathbb{E}(\text{Y})_{I} &= \frac{\sigma_{w}}{\sqrt{d_{0}}}W^{1}X_{\cdot i} + \sigma_{b} b^{1} = 0 \\
\mathrm{Cov}(\text{Y})_{IJ} &= \mathbb{E} \biggl[ \left( \frac{\sigma_{w}}{\sqrt{d_{0}}}W^{1}X_{\cdot i} + \sigma_{b} b^{1} \right) \left(\frac{\sigma_{w}}{\sqrt{d_{0}}}W^{1}X_{\cdot j} + \sigma_{b}b^{1} \right)^{T} \biggl] \\
&= \frac{\sigma_{w}^{2}}{d_{0}} \mathbb{E} \biggl[ W^{1} X_{\cdot i} X_{\cdot j}^{T} W^{1^{T}} \biggr] + \sigma^{2}_{b} \mathbb{E} \biggl[b^{1} b^{1^{T}} \biggr] \\
&= \frac{\sigma_{w}^{2}}{d_{0}} I_{d_{1}}   \left< X_{\cdot i}, X_{\cdot j} \right> + I_{d_{1}}\sigma^{2}_{b} \\
&=I_{d_{1}} \left(  \frac{\sigma_{w}^{2}}{d_{0}} \left< X_{\cdot i}, X_{\cdot j} \right> + \sigma^{2}_{b} \right)
\intertext{Base case for the complete dataset:}
\mathbb{E}(\vect(Y^{1})) &= \frac{\sigma_{w}}{\sqrt{d_{0}}} \vect(W^{1}X) + \sigma_{b} \vect(B^{1}) = 0 \\
\mathrm{Cov}(\vect(Y^{1})) &= \mathbb{E} \biggl[ \biggl( \frac{\sigma_{w}}{\sqrt{d_{0}}}\vect(W^{1}X) + \sigma_{b} \vect(B^{1} \biggr)   \biggl( \frac{\sigma_{w}}{\sqrt{d_{0}}}\vect(W^{1}X) + \sigma_{b} \vect(B^{1}))^{T} \biggr) \biggr] \\
&= \frac{\sigma_{w}^{2}}{d_{0}} \mathbb{E} \biggl[ \vect(W^{1}X) \vect(W^{1}X)^{T} \biggr] + \sigma^{2}_{b}\mathbb{E} \biggl[ \vect(B^{1}) \vect(B^{1})^{T} \biggr] \\
&= \frac{\sigma_{w}^{2}}{d_{0}}  \mathbb{E} \biggl[ (I_{n} \otimes W^{1}) \vect(X) \vect(X)^{T}(I_{n} \otimes W^{1^{T}}) \biggr] + \sigma^{2}_{b} (\mathds{1}_{n} \otimes I_{d_{1}}) \\
&= \biggl( \frac{\sigma^{2}_{w}}{d_{0}} X^{T}X +  \sigma^{2}_{b} \biggr) \otimes I_{d_{1}}
\intertext{The induction step for one, respectively two data samples.}
\mathbb{E}(\text{Y}^{l+1})_{I} &= 0 \\
\mathrm{Cov}(\text{Y}^{l+1})_{IJ} &= \mathbb{E}
\bigg[ \biggl( \frac{\sigma_{w}}{\sqrt{d_{l}}}W^{l+1} G^{l}_{\cdot i}) + \sigma_{b} b^{l+1} \biggr) \biggl( \frac{\sigma_{w}}{\sqrt{d_{l}}}W^{l+1} \sigma(Y^{l}_{\cdot i}) + \sigma_{b} b^{l+1}) \biggr)^{T} \biggr] \\
&= \frac{\sigma_{w}^{2}}{d_{l}} \mathbb{E} \biggr[ W^{l+1} \sigma(Y^{l}_{\cdot i})  \sigma(Y^{l}_{\cdot j})^{T} W^{l+1^{T}} \biggr] + \sigma^{2}_{b}I_{d_{l}} \\
 \frac{\sigma_{w}^{2}}{d_{l}}\mathbb{E}(W^{l+1} &\sigma(Y^{l}_{\cdot i})  \sigma(Y^{l}_{\cdot j} W^{l+1^{T}})^{T})_{uv} = \frac{\sigma^{2}_{w}}{d_{l}} \tr(\sigma(Y^{l}_{\cdot i}) \sigma(Y^{l}_{\cdot j})^{T}) \delta_{uv} \tag{\theequation.1}\label{myeq1} \\
 & \delta_{uv} \frac{\sigma^{2}_{w}}{d_{l}} \sum_{n=1}^{d_{l}}  \sigma(Y^{l}_{nj}) \sigma(Y^{l}_{ni}) \overset{\overset{P}{d_{l} \rightarrow \infty}}{\longrightarrow}   \sigma^{2}_{w}\delta_{uv} \underset{s,t \sim N(0, \biggl( \begin{smallmatrix} 
\Lambda^{l-1}_{ii} & \Lambda^{l-1}_{ij} \\
\Lambda^{l-1}_{ji} & \Lambda^{l-1}_{jj} 
\end{smallmatrix}  \biggr)) }{\mathbb{E}} \bigl[\sigma(s)\sigma(t) \bigr]\tag{\theequation.2}\label{myeq2}
\end{align*}
In ~\eqref{myeq1} we have used the quadratic form for random variables.
In ~\eqref{myeq2} we have used the fact that $Y^{l}_{ni}$ and $Y^{l}_{mi}$ with $n \neq m$ are i.i.d. random variables. That also becomes apparent by observing that $\mathrm{Cov}(Y^{l}_{\cdot i}, Y^{l}_{ \cdot j}$) is diagonal.
And finally for the whole dataset:
\begin{align*}
\mathbb{E}(\vect(Y^{l+1})) &= \frac{\sigma_{w}}{\sqrt{d_{l}}} \vect(W^{l+1}G^{l}) + \sigma_{b} \vect(B^{l+1}) = 0 \\
\mathrm{Cov}(\vect(Y^{l+1})) &= \mathbb{E} \biggr[ \biggl( \frac{\sigma_{w}}{\sqrt{d_{l}}}\vect(W^{l+1}G^{l}) + \sigma_{b} \vect(B^{l+1} \biggl) \biggr( \frac{\sigma_{w}}{\sqrt{d_{l}}}\vect(W^{l+1}G^{l}) + \sigma_{b} \vect(B^{l+1})^{T} \biggl) \biggr] \\
&= \frac{\sigma_{w}^{2}}{d_{l}} \mathbb{E} \bigl[ \vect(W^{l+1}G^{l}) \vect(G^{l^{T}}W^{l+1^{T}} \bigl] + \sigma^{2}_{b} (\mathds{1}_{n} \otimes I_{d_{l}}) \\
&= \frac{\sigma_{w}^{2}}{d_{l}}  \mathbb{E} \bigl[ (I_{n} \otimes W^{l+1}) \vect(G^{l}) \vect(G^{l})^{T}(I_{n} \otimes W^{l+1^{T}} ) \bigr] + \sigma^{2}_{b} (\mathds{1}_{n} \otimes I_{d_{l+1}})\tag{\theequation.2}\label{nngp-eq2} \\
&\frac{\sigma_{w}^{2}}{d_{l}}  \mathbb{E} \bigl[ (I_{n} \otimes W^{l+1}) \vect(G^{l}) \vect(G^{l})^{T}(I_{n} \otimes W^{l+1^{T}}) \bigr] + \sigma^{2}_{b} (\mathds{1}_{n} \otimes I_{d_{l}}) \\
&=  \left( \sigma_{w}^{2} \begin{pmatrix} 
    \mathbb{E} \bigl[ \sigma(u_{1}) \sigma(u_{1}) \bigr] & \mathbb{E} \bigl[ \sigma(u_{1}) \sigma(u_{2}) \bigr] \dots  & \mathbb{E} \bigl[ \sigma(u_{1}) \sigma(u_{n}) \bigr]\\
    \vdots & \ddots & \vdots\\
    \mathbb{E} \bigl[ \sigma(u_{n}) \sigma(u_{1}) \bigr] & \dots  & \mathbb{E} \bigl[ \sigma(u_{n}) \sigma(u_{n}) \bigr]
    \end{pmatrix} + \sigma^{2}_{b} \right) \otimes I_{d_{l+1}} \\[14pt]
    &\text{ with } u \sim N(0, \Lambda^{L-1})
\end{align*}

In ~\eqref{nngp-eq2} we have used the property of the Kronecker Product, namely $\vect(AXB) = (B^{T} \otimes A)\vect(X)$.
\end{proof}

\subsubsection{Neural Tangent Kernel (Lemma \ref{ntk:lemma} and Theorem~\ref{th:nngp})}\label{ntk1}
The following lemma describes how the output of a Fully-Connected Neural Net trained with Gradient Descent with an infinitesimally small learning rate can be interpreted as Kernel Regression with the Neural Tangent Kernel.
\begin{lemma}[From \citep{arora2019exact}]\label{ntk:lemma}
Consider minimizing the squared loss $\ell(\theta)$ by Gradient Descent with infinitesimally small learning rate\footnote{\url{https://en.wikipedia.org/wiki/Euler_method}}: 
$\frac{d\theta(t)}{dt} = - \nabla \ell(\theta(t))$. Then the output of the Network $F(\theta, X)$ evolves as:
\begin{align*}
    \frac{d u(t)}{dt} &= - H(t)\cdot \vect(u(t) - Y)  \\
    \text{with } u(t) &= F(\theta(t), X) \\
    \text{and } H(t) &=\p{\vect(F(\theta(t), X)}{\theta}  \left( \p{\vect(F(\theta(t), X)}{\theta}  \right)^{T}
\end{align*}
\begin{proof}
    The parameters $\theta$ evolve according to the differential equation 
    \begin{align*}
        \frac{d\theta(t)}{dt} &= - \nabla \ell(\theta(t)) = \left(\p{\vect(F(\theta(t), X)}{\theta}\right)^{T}  \vect(F(\theta(t), X) - Y)^{T} 
        \intertext{where $t>0$ is a continous time index. Using this equation, the evolution of the Network output $F(\theta(l), X)$ can be written as} 
        \frac{dF(\theta(t), X}{dt} &=   \p{\vect(F(\theta(t), X)}{\theta} \left(\p{\vect(F(\theta(t), X)}{\theta}\right)^{T}  \vect(F(\theta(t), X) - Y)^{T} 
        \intertext{Rewriting it using $\emph{u}(t) = \vect(F(\theta(t), X))$}
        \frac{d\emph{u}(t)}{dt} &= - \emph{H}(t)\cdot \vect(\emph{u}(t) - Y) 
    \end{align*}
\end{proof}
\end{lemma} \noindent
When the width of the network is allowed to go to infinity, it can be shown that the matrix $H(t)$ remains \emph{constant} during training, i.e. equal to $H(0)$. Moreover, under the random initialization of the parameters defined in \ref{fcn-ef}, the matrix $H(0)$ converges in probability to a certain deterministic Kernel Matrix $H^{*}$ which is called the \emph{Neural Tangent Kernel} $\Theta$ evaluated on the training data. If $H(t) = H^{*}$ for all $t$, then equation becomes 
\begin{align*}
\frac{d u(t)}{dt} &= - H^{*} \cdot \vect(u(t) - Y) 
\end{align*} 
which is identical to the dynamics of \emph{Kernel Regression} under Gradient Glow, for which at time $t \rightarrow \infty$ the final prediction function is 
\begin{align*}
    F^{*}(X_{test}) = H^{*}_{test, train} (H_{train, train}^{*})^{-1}y
\end{align*}
In the transductive setting $H_{train, train}^{*}$ implicitly accounts for incorporating the test data during training. 
To calculate the closed form expression for the NTK of a FCN of depth $l$ we have to simplify:
\begin{align*}
    \Theta^{l} =  \p{\vect(F^{l}(\theta, X))}{\theta} \left( \p{\vect(F^{l}(\theta, X))}{\theta} \right)^{T}
\end{align*}
Using our model definition, this results in calculating
\begin{align}
\Theta^{L} =  \sum_{h=1}^{H} \p{\vect(F^{L}(\theta, X))}{\vect(W^{h})} \left( \p{\vect(F(\theta, X)}{\vect(W^{h})} \right) ^{T}  +
\p{\vect(F(\theta, X))}{\vect(b^{h})} \left( \p{\vect(F^{L}(\theta, X)}{\vect(b^{h})} \right) ^{T}
\end{align}

\subsubsection*{Neural Tangent Kernel (Theorem~\ref{th:ntk})}\label{th:ntkproof}
\thnntk*
\begin{proof}
Proof is by induction, similar to the NNGP derivation.\\
Base Case:
\begin{align*}
    \Theta^{1} &= \Lambda{^1} = \p{\vect(F^{1}(\theta, X))}{\vect(W^{1})} \left( \p{\vect(F^{1}(\theta, X)}{\vect(W^{1})} \right) ^{T} + \p{\vect(F^{1}(\theta, X))}{\vect(b^{1})} \left( \p{\vect(F^{1}(\theta, X)}{\vect(b^{1})} \right) ^{T} \\
   &= (\frac{\sigma^{2}_{w}}{d_{0}}X^{T}X + {\sigma^{2}_{b}}) \otimes I_{d_{1}} \\
\intertext{Induction Step:}
\Theta^{l+1} &= \sum_{h=1}^{l+1} \p{\vect(F^{l+1}(\theta, X))}{\vect(W^{h})} \left( \p{\vect(F^{l+1}(\theta, X)}{\vect(W^{h})} \right)^{T} \\ 
& = \underbrace{\p{\vect(F^{l+1}(\theta, X))}{\vect(W^{l+1})} \left( \p{\vect(F^{l+1}(\theta, X))} {\vect(W^{l+1})} \right) ^{T} + \p{\vect(F^{l+1}(\theta, X))}{\vect(b^{l+1})} \left( \p{\vect(F^{l+1}(\theta, X))} {\vect(b^{l+1})} \right) ^{T}}_{B} \\
& + \underbrace{\sum_{h=1}^{l}\p{\vect(F^{l+1}(\theta, X))}{\vect(W^{h})} \left( \p{\vect(F^{l+1}(\theta, X))} {\vect(W^{h})} \right) ^{T} + \p{\vect(F^{l+1}(\theta, X))}{\vect(b^{h})} \left( \p{\vect(F^{l+1}(\theta, X))} {\vect(b^{h})} \right) ^{T}}_{\Gamma} \\
    B &= \frac{\sigma^{2}_{w}}{d_{l}} (G^{l^{T}}G \otimes I_{d_{l+1}}) + \sigma^{2}_{b}(\mathds{1}_{n} \otimes I_{d_{l+1}}) = (\frac{\sigma^{2}_{w}}{d_{l}}G^{l^T}G^{l} + \sigma^{2}_{b}) \otimes I_{d_{l+1}} \\
    &(\frac{\sigma^{2}_{w}}{d_{l}}G^{l^T}G^{l})_{ij} = \frac{\sigma^{2}_{w}}{d_{l}} \sum_{m}^{d_{l}}  G^{l}_{im} G^{l}_{jm} \overset{\overset{P}{d_{l} \rightarrow \infty}}{\longrightarrow}  \sigma^{2}_{w} \underset{u,v \sim N(0, \biggl( \begin{smallmatrix} 
\Lambda^{l-1}_{ii} & \Lambda^{l-1}_{ij} \\
\Lambda^{l-1}_{ji} & \Lambda^{l-1}_{jj} 
\end{smallmatrix}  \biggr)) }{\mathbb{E}} \bigl[ \sigma(u)\sigma(v) \bigr]
\end{align*}
Where we used the fact that the output of the previous layer goes to infinity (so $d_{l} \rightarrow \infty$)
\begin{align*}
    \implies B = \Lambda^{L} \otimes I_{d_{l+1}}
\end{align*}
\begin{align*}
\Gamma &= 
\sum_{h=1}^{l} \left( \p{\vect(F^{l+1}(\theta, X))}{\vect(G^{l})} \p{\vect(G^{l})}{\vect(F^{l})} \p{\vect(F^{l}(\theta, X))}{\vect(W^{l})}   \right) \left( \p{\vect(F^{l+1}(\theta, X))}{\vect(G^{l})} \p{\vect(G^{l})}{\vect(F^{l})}  \p{\vect(F^{l}(\theta, X))}{\vect(W^{l})}  \right)^{T}
 \\
&+ 
\left( \p{\vect(F^{l+1}(\theta, X))}{\vect(G^{l})} \p{\vect(G^{l})}{\vect(F^{l})} \p{\vect(F^{l}(\theta, X))}{\vect(b^{l})}   \right) \left( \p{\vect(F^{l+1}(\theta, X))}{\vect(G^{l})} \p{\vect(G^{l})}{\vect(F^{l})}  \p{\vect(F^{l}(\theta, X))}{\vect(b^{l})}  \right)^{T}
 \\
&= \p{\vect(F^{l+1}(\theta, X))}{\vect(G^{l})} \p{\vect(G^{l})}{\vect(F^{l})} (\Theta^{l} \otimes I_{d_{l}})  \left( \p{\vect(F^{l+1}(\theta, X))}{\vect(G^{l})} \p{\vect(G^{l})}{\vect(F^{l})}   \right)^{T}\\
&= \frac{\sigma^{2_{w}}}{d_{l}}(I_{n} \otimes W^{l+1}) \Sigma^{l} (\Theta^{l} \otimes I_{d}) \Sigma^{l}  (I_{n} \otimes W^{l+1^{T}}) \\\\
T_{IJ} &=  \left(\frac{\sigma_{w}^{2}}{d_{l}} (I_{n} \otimes W^{l+1}) \Sigma^{l} (\Theta^{l} \otimes I_{d}) \Sigma^{l} (I_{n} \otimes W^{l+1^{T}}) \right)_{IJ} = \frac{\sigma_{w}^{2}}{d_{l}} W^{l+1} \Sigma^{l}_{II} \Theta^{l}_{ij} \Sigma^{l}_{JJ} W^{l+1^{T}} \\\\
&= \Theta^{l}_{ij} \frac{\sigma_{w}^{2}}{d_{l}}  W^{l+1} \Sigma^{l}_{II} \Sigma^{l}_{JJ} W^{l+1^{T}} \\
(\frac{\sigma_{w}^{2}}{d_{l}} W&^{l+1} \Sigma^{l}_{II} \Sigma^{l}_{JJ} W^{l+1^{T}})_{uv} = \frac{\sigma_{w}^{2}}{d_{l}} \sum_{st}^{d_{l}}\Sigma^{l}_{II_{ss}}\Sigma^{l}_{JJ_{tt}} W_{us} W_{vt} \overset{\overset{P}{d_{l} \rightarrow \infty}}{\longrightarrow}  \delta_{uv} \sigma^{2}_{w} \underset{s,t \sim N(0, \biggl( \begin{smallmatrix} 
\Lambda^{l-1}_{ii} & \Lambda^{l-1}_{ij} \\
\Lambda^{l-1}_{ji} & \Lambda^{l-1}_{jj} \\
\end{smallmatrix}  \biggr)) }{\mathbb{E}} \bigl[\dot{\sigma}(s) \dot{\sigma}(t) \bigr]
\end{align*}
Convergence in probability follows from the Law of Large Numbers. The righthandside follows from the definition of $\Sigma^{l}$ and the fact that the output of layer $l-1$ in its infinite width is a Gaussian Process.
\begin{align*}
    \implies T &= (\Theta \odot \dot{\Lambda}^{L}) \otimes I_{d_{l+1}} \\
     \implies \Theta^{l+1}  &= B + \Gamma  =  \Lambda^{l} +  \left( \dot{\Lambda}^{l}  \odot \Theta^{l} \right)  
\end{align*}
\end{proof}

\subsection{NTK \& GP for Graph Neural Network}\label{gnn123}
We copied the definitions of the Graph Neural Network from the main section.

\defgnn*

\subsubsection*{Derivatives} \label{gnnder}
The derivative of $F(\theta, X)$ with respect to $W^{l}$ is
\begin{align*}
\p{\vect(F(\theta, X))}{\vect(W^{l})} = \p{\vect(F^{L})}{\vect(G^{L-1})} \p{\vect(G^{L-1})}{\vect(F^{L-1)}} \cdots \p{\vect(F^{l+1})}{\vect(G^{l})}
\p{\vect(G^{l})}{\vect(F^{l})} \p{\vect(F^{l})}{\vect(W^{l})}  
\end{align*}
Using
\begin{align*}
\p{\vect(F^{h})}{\vect(G^{h-1})} &=  \frac{\sigma_{w}}{\sqrt{d_{h-1}}} (A \otimes W^{h}) \\
\Sigma^{h} := \p{\vect(G^{h})}{\vect(F^{h})} &= \diag(\vect(\dot{\sigma}(F^{h})) \\
\p{\vect(F^{h})}{\vect(W^{h})} &= \frac{\sigma_{w}}{\sqrt{d_{h-1}}}( A G^{h-1^{T}} \otimes I_{d_{h}}) \\
\p{\vect(F^{h})}{b^{h}} &= \sigma_{b} (A \mathbf{1}_{n} \otimes I_{d_{h}})
\end{align*}

\subsubsection{Graph Neural Network Gaussian Process (Theorem \ref{th:gnngp})} \label{th:gnngp1}
\thngnngp*
\begin{proof}
Proof is by induction. We define $\text{Y}^{L}= F^{L}(X)$.
\begin{align*}
\mathbb{E}(\vect(Y^{1})) &= \frac{\sigma_{w}}{\sqrt{d_{0}}} \vect \left( \biggl( \frac{\sigma_{w}}{\sqrt{d_{0}}}W^{1}X + \sigma_{b}B^{1} \biggr) A^{T} \right) = 0 \\
\mathrm{Cov}(\vect(Y^{1})) &= \mathrm{Cov} \bigl[ (A \otimes I_{d_{1}})\vect(\frac{\sigma_{w}}{\sqrt{d_{0}}}W^{1}X + \sigma_{b}B^{1}) \bigr] \\
&= (A \otimes I_{d_{1}})  \mathbb{E} \bigl[ \vect(\frac{\sigma_{w}}{\sqrt{d_{0}}}W^{1}XA^{T} + \sigma_{b}B^{1}) \bigr] (A^{T} \otimes I_{d_{1}})  \\
&= (A \Lambda^{1} A^{T}) \otimes I_{d_{1}}
\end{align*}
The Base Case is reduced to the Base Case of NNGP using a property of the Kronecker Product, namely $(A \otimes B)(C \otimes D) = (AC \otimes BD)$.\\
Induction Step is again just an application of the Kronecker Property and follows the NNGP proof in \ref{th:nngp}.
\begin{align*}
\mathbb{E}(\vect(Y^{l+1})) &=  0 \\
\mathrm{Cov}(\vect(Y^{l+1})) &= (A \Lambda^{l+1} A^{T}) \otimes I_{d_{l+1}}
\end{align*}
\end{proof}

\subsubsection{Graph Neural Tangent Kernel (Theorem~\ref{th:gntk})} \label{th:gntkproof}
\thngntk*
\begin{proof}
Proof is by induction, similar to the NNGP derivation.\\
Base Case:
\begin{align*}
   & \Theta^{1} = \Lambda{^1} = \p{\vect(F^{1}(\theta, X))}{\vect(W^{1})} \left( \p{\vect(F^{1}(\theta, X)}{\vect(W^{1})} \right) ^{T} + \p{\vect(F^{1}(\theta, X))}{\vect(b^{1})} \left( \p{\vect(F^{1}(\theta, X)}{\vect(b^{1})} \right) ^{T} \\
   &= \bigl( A \bigl( \frac{\sigma^{2}_{w}}{d_{0}}X^{T}X + {\sigma^{2}_{b}} \bigr) A^{T} \bigr) \otimes I_{d_{1}}
\end{align*}
Induction Step:
\begin{align*}
\Theta^{l+1} &= \sum_{h=1}^{l+1} \p{\vect(F^{l+1}(\theta, X))}{\vect(W^{h})} \left( \p{\vect(F^{l+1}(\theta, X)}{\vect(W^{h})} \right)^{T} \\ 
& = \underbrace{\p{\vect(F^{l+1}(\theta, X))}{\vect(W^{l+1})} \left( \p{\vect(F^{l+1}(\theta, X))} {\vect(W^{l+1})} \right) ^{T} + \p{\vect(F^{l+1}(\theta, X))}{\vect(b^{l+1})} \left( \p{\vect(F^{l+1}(\theta, X))} {\vect(b^{l+1})} \right) ^{T}}_{B} \\
& + \underbrace{\sum_{h=1}^{l}\p{\vect(F^{l+1}(\theta, X))}{\vect(W^{h})} \left( \p{\vect(F^{l+1}(\theta, X))} {\vect(W^{h})} \right) ^{T} + \p{\vect(F^{l+1}(\theta, X))}{\vect(b^{h})} \left( \p{\vect(F^{l+1}(\theta, X))} {\vect(b^{h})} \right) ^{T}}_{\Gamma} \\
B &= \biggl( A \biggl( \frac{\sigma^2_{w}}{d_{l}}G^{l^T}G^{l} + \sigma^{2}_{b}  \biggr) A^{T} \biggr) \otimes I_{d_{l+1}} \\
 \end{align*}
We can now proceed similar to the previous NTK derivations (see \ref{th:ntk}) and therefore will skip the parts which stay the same.
\begin{align*}
    \implies B = A \left( \Lambda^{L} \otimes I_{d_{l+1}} \right) A^{T}
\end{align*}

\begin{align*}
 \Gamma &=
\sum_{h=1}^{l} \left( \p{\vect(F^{l+1}(\theta, X))}{\vect(G^{l})} \p{\vect(G^{l})}{\vect(F^{l})} \p{\vect(F^{l}(\theta, X))}{\vect(W^{l})}   \right) \left( \p{\vect(F^{l+1}(\theta, X))}{\vect(G^{l})} \p{\vect(G^{l})}{\vect(F^{l})}  \p{\vect(F^{l}(\theta, X))}{\vect(W^{l})}  \right)^{T}
 \\
&+ 
\left( \p{\vect(F^{l+1}(\theta, X))}{\vect(G^{l})} \p{\vect(G^{l})}{\vect(F^{l})} \p{\vect(F^{l}(\theta, X))}{\vect(b^{l})}   \right) \left( \p{\vect(F^{l+1}(\theta, X))}{\vect(G^{l})} \p{\vect(G^{l})}{\vect(F^{l})}  \p{\vect(F^{l}(\theta, X))}{\vect(b^{l})}  \right)^{T}
 \\
&= \frac{\sigma^{2_{w}}}{d_{l}}(A \otimes W^{l+1}) \Sigma^{l} (\Theta^{l} \otimes I_{d}) \Sigma^{l}  (A^{T} \otimes W^{l+1^{T}}) \\
&= \frac{\sigma^{2_{w}}}{d_{l}}  (A \otimes I_{d_{l}})(I_{n} \otimes W^{l+1}) \Sigma^{l} (\Theta^{l} \otimes I_{d}) \Sigma^{l}  (I_{n} \otimes W^{l+1^{T}} (A^{T} \otimes I_{d_{l}}))
\end{align*}

We reduced the proof to the NTK proof using property of the Kronecker Product in the last line.
\begin{align*}
    \implies T = ( A(\Theta \odot \dot{\Lambda}^{L})A^{T}) \otimes I_{d_{l+1}} 
\end{align*}
\end{proof}
\subsection{NTK \& GP for Graph Neural Network with Skip-Concatenate Connections} \label{sgnn123}
The model definition is copied from Section \ref{sgnn-def} for the readers convenience.

\defsgnn*

The derivatives are similar to the Graph Neural Network derivatives (see \ref{gnnder}).
Only $\Sigma^{h}$ is different.
\begin{align*}
&\Sigma^{h} := \p{\vect(G^{h})}{\vect(F^{h})} = \begin{pmatrix}
    \diag(\vect\dot{\sigma}(F^{h}_{ \cdot 1})) & 0 & 0 & \dots & 0\\
    I_{d_{h}} & 0 & 0 & \dots & 0 \\ 
    0 &  \diag(\vect\dot{\sigma}(F^{h}_{ \cdot 2})) & 0 & \dots & 0 \\
    0 & I_{d_{h}} & 0 & \dots  \\
     0 & 0 & \ddots \\
     \vdots  & \vdots & \ddots \\
     0 & 0 & &  \dots & \diag(\vect\dot{\sigma}(F^{h}_{ \cdot n})) \\
      0 & 0 & 0 & \dots & I_{d_{h}}   \\
\end{pmatrix}  \in \mathbb{R}^{2d_{l}n \times 2d_{l}n} \\
&\Sigma^{h} \text{ is block-diagonal, with each block } \Sigma^{h}_{II} = \begin{pmatrix} 
\diag(\vect\dot{\sigma} (F^{h}_{ \cdot i})) \\
 I_{d_{h}}
\end{pmatrix}
\end{align*}

\subsubsection{Gaussian Process (Theorem~\ref{th:sgnngp})} \label{th:sgnn-proof}
\thnsgnnp*
\begin{proof}
Prove is by induction. We define $\text{Y}^{L}= F^{L}(X)$.
The Base Case doesn't change, so
\begin{gather*}
    \mathbb{E}(\vect(Y^{1})) =  0 \\
\mathrm{Cov}(\vect(Y^{1})) = (A \Lambda^{1} A^{T}) \otimes I_{d_{1}}
\intertext{The Induction Step:}
\mathbb{E}(\vect(Y^{l+1})) = 0 \\
\mathrm{Cov}(\vect(Y^{l+1})) = \mathrm{Cov} \bigl[ (A \otimes I_{d_{l+1}})\vect(\frac{\sigma_{w}}{\sqrt{d_{0}}}W^{l+1}G^{l}A^{T} + \sigma_{b}B^{l+1}) \bigr] \\
= (A \otimes I_{d_{l+1}})  \mathrm{Cov}  \bigl[ \vect(\frac{\sigma_{w}}{\sqrt{d_{0}}}W^{l+1}G^{l} + \sigma_{b}B^{1}) \bigr] (A^{T} \otimes I_{d_{l+1}})  \\
= (A \otimes I_{d_{l+1}}) \left(  \frac{\sigma_{w}^{2}}{d_{l}} \mathbb{E} \biggl[ (I_{n} \otimes W^{l+1}) \vect(G^{l}) \vect(G^{l})^{T} (I_{n} \otimes W^{l+1^{T}}) \biggr]  + \sigma_{b}^{2} (\mathds{1}_{n} \otimes \mathbf{1}_{d_{l+1}}) \right) (A^{T} \otimes I_{d_{l+1}})
\end{gather*}
\begin{gather*}
\frac{\sigma_{w}^{2}}{2d_{l}} \mathbb{E} \biggl[ (I_{n} \otimes W^{l+1}) \vect(G^{l}) \vect(G^{l})^{T} (I_{n} \otimes W^{l+1^{T}}) \biggr]_{IJ_{uv}} = \frac{\sigma_{w}^{2}}{2d_{l}}\mathbb{E} \biggl[ W_{u} (\vect(G^{l}) \vect(G^{l})^{T})_{IJ} W_{v}^{T} \biggr] \\
\intertext{$Z := \left( \vect(G^{l}) \vect(G^{l})^{T} \right)_{IJ}$ and $Z_{11}$ is the upper left block matrix of size $\mathbb{R}^{d_{l} \times d_{l}}$.}
\intertext{
and $W^{l+1}_{u} = \biggl[ \WW{l+1}_{u},  \WWW{l+1}_{u} \biggr]$ with $\WW{l+1}_{u},  \WWW{l+1}_{u} \in \mathbb{R}^{1 \times d_{l}}$, so we split $W^{l+1}_{u}$ into two parts, each of length $d_{l}$.
}\\
\frac{\sigma_{w}^{2}}{2d_{l}}\mathbb{E} \biggl[ W_{u} (\vect(G^{l}) \vect(G^{l})^{T})_{IJ} W_{v}^{T} \biggr]\\
=\frac{\sigma_{w}^{2}}{2d_{l}} \mathbb{E} \biggl[ \WW{l+1}_{u}  Z_{11}    \WW{l+1}_{v}  +    \WW{l+1}_{u}  Z_{12}    \WWW{l+1}_{v}  +    \WWW{l+1}_{u}  Z_{21}    \WW{l+1}_{v} +  \WWW{l+1}_{u}  Z_{22}    \WWW{l+1}_{v}  \biggr] \\
=\delta_{uv}\frac{\sigma_{w}^{2}}{2d_{l}} \mathbb{E} \biggl[\WW{l+1}_{u}  Z_{11}    \WW{l+1}_{v} + \WWW{l+1}_{u}  Z_{22}    \WWW{l+1}_{v} \biggr]\\
= \delta_{uv}\frac{\sigma_{w}^{2}}{2d_{l}} \left(\sum_{s}^{d_{l}} \sigma(Y^{l})_{is} \sigma(Y^{l})_{js} + \sum_{s}^{d_{l}} (Y^{l}_{is} Y^{l}_{js}) \right)  =  \delta_{uv}\frac{\sigma_{w}^{2}}{2d_{l}} \left(\sum_{s}^{d_{l}} \sigma(Y^{l})_{is} \sigma(Y^{l})_{js} +  (Y^{l}_{is} Y^{l}_{js}) \right) \\
\overset{\overset{P}{d_{l} \rightarrow \infty}}{\longrightarrow}  \delta_{uv} \frac{1}{2}\biggl(\underset{s,t \sim N(0, \biggl( \begin{smallmatrix} 
\Lambda^{l-1}_{ii} & \Lambda^{l-1}_{ij} \\
\Lambda^{l-1}_{ji} & \Lambda^{l-1}_{jj} 
\end{smallmatrix}  \biggr)) }{\mathbb{E}} \bigl[\sigma(s)\sigma(t) \bigr] + \mathbb{E}\bigl[ st \bigr] \biggr)
\end{gather*}
Convergence in probability follows from the Law of Large Numbers. The righthandside follows from the definition of $\Sigma^{l}$ and the fact that the output of layer $l-1$ is a from a Gaussian Process.
\end{proof}

\subsubsection{Neural Tangent Kernel (Theorem~\ref{th:sgntk})} \label{th:sgnnntk-proof}
\thnsgntk*

\begin{proof}
Proof is by Induction, similar to the NNGP derivation. The Base Case is just a repetition of \ref{th:gntk} as $F^{1}$ has no Skip-Concatenate Connections. We will skip parts which are same as the NTK/ GNTK derivations.
Induction Step:
\begin{align*}
\Theta^{l+1} &= B + \Gamma \\
B &= \biggl( A \biggl( \frac{\sigma_{w}^{2}}{2d_{l}}G^{l^T}G^{l} + \sigma^{2}_{b} \mathds{1}_{n} \biggr) A^{T} \biggr) \otimes I_{d_{l+1}} \\
\hspace*{-1.5cm}\frac{\sigma_{w}^{2}}{2d_{l}}G^{l^T}G^{l} &\overset{\overset{P}{d_{l} \rightarrow \infty}}{\rightarrow} \sigma^{2}_{w} 
\begin{psmallmatrix} 
    \frac{1}{2 }\biggl( \mathbb{E}\bigl[ \sigma(u_{1}) \sigma(u_{1}) \bigr] + \mathbb{E} \bigl[ u_{1} u_{1} \bigr] \biggr) &   \frac{1}{2 }\biggl(\mathbb{E} \bigl[ \sigma(u_{1}) \sigma(u_{2}) \bigr] + \mathbb{E} \bigl[ u_{1} u_{2} \bigr]\biggr) \dots  &   \frac{1}{2 }\biggl(\mathbb{E} \bigl[ \sigma(u_{1}) \sigma(u_{n}) \bigr] + \mathbb{E} \bigl[ u_{1} u_{n} \bigr] \biggr)\\
    \vdots & \ddots & \vdots\\
     \frac{1}{2 }\biggl(  \mathbb{E} \bigl[ \sigma(u_{n}) \sigma(u_{1}) \bigr] + \mathbb{E} \bigl[ u_{n} u_{1} \bigr]\biggr)  & \dots  &  \frac{1}{2 }\biggl( \mathbb{E} \bigl[ \sigma(u_{n}) \sigma(u_{n}) \bigr] + \mathbb{E} \bigl[ u_{n} u_{n} \bigr]\biggr)
    \end{psmallmatrix}
\end{align*}
\begin{align*}
    \implies B = \Lambda^{L} \otimes I_{d_{l+1}}
\end{align*}

\begin{align*}
&\Gamma = \frac{\sigma^{2_{w}}}{2d_{l}}  (A \otimes I_{d_{l}})(I_{n} \otimes W^{l+1}) \Sigma^{l} (\Theta^{l} \otimes I_{d}) \Sigma^{l}  (I_{n} \otimes W^{l+1^{T}}) (A^{T} \otimes I_{d_{l}})) \\
&\frac{\sigma^{2}_{w}}{2d_{l}} (I_{n} \otimes W^{l+1}) \Sigma^{l} (\Theta^{l} \otimes I_{d}) \Sigma^{l}  (I_{n} \otimes W^{l+1^{T}})_{IJ_{uv}} = \Theta_{ij}  \frac{\sigma^{2}_{w}}{2d_{l}}    W^{l+1}_{u} \Sigma^{l}_{II} \Sigma^{l^{T}}_{JJ} (W^{l+1}_{v})^{T} \\
\intertext{$Z := \left( \vect(G^{l}) \vect(G^{l})^{T} \right)_{IJ}$ and $Z_{11}$ is the upper left block matrix of size $\mathbb{R}^{d_{l} \times d_{l}}$.}
\intertext{
and $W^{l+1}_{u} = \biggl[ \WW{l+1}_{u},  \WWW{l+1}_{u} \biggr]$ with $\WW{l+1}_{u},  \WWW{l+1}_{u} \in \mathbb{R}^{1 \times d_{l}}$, so we split $W^{l+1}_{u}$ into two parts, each of length $d_{l}$.
}\\
&= \Theta_{ij} \frac{\sigma_{w}^{2}}{2d_{l}} \left( \WW{l+1}_{u}  Z_{11}    \WW{l+1}_{v}  +    \WW{l+1}_{u}  Z_{12}    \WWW{l+1}_{v}  +    \WWW{l+1}_{u}  Z_{21}    \WW{l+1}_{v} +  \WWW{l+1}_{u}  Z_{22}    \WWW{l+1}_{v}  \right) \\
& \overset{\overset{P}{d_{l} \rightarrow \infty}}{\longrightarrow} \delta_{uv}\Theta_{ij} \sigma_{w}^{2} \frac{1}{2} \biggl(
\mathbb{E}\bigl[ \dot{\sigma}(s)\dot{\sigma}(t) \bigr] + 1 \biggr)
\end{align*}
Which follows the same idea for the proof of the NTK.
\end{proof}

\subsection{NTK \& GP for Graph Attention Neural Network}
The model definition is copied from Section 3.3 for the readers convenience.
\defgat*
\subsubsection*{Derivatives}
\begin{align*}
\SSSigma{l} &:= \pv{G^{l}}{F^{l}} = \diag(\vect(\dot{\sigma}_{2}(F^{l}))) \\
\p{\vect(F^{l})}{\vect(W^{l})} &=    \frac{\sigma_{w}}{\sqrt{d_{l-1}}} \left[F^{l1}, \cdots, F^{lH} \right]  \otimes I_{d_{l}}  \\
\p{\vect(F^{l})}{\vect(\sigma_{1}(L^{l,h}))} &=  I_{n} \otimes \left( \frac{\sigma_{w}}{\sqrt{d_{l-1}}} W^{l,h}G^{l-1} \right) \\
\SSigma{l,h} &:= \p{\vect(\sigma_{1}(L^{l,h}))}{\vect(L^{l,h})} = \diag(\vect(\dot{\sigma}_{1}(L^{l,h}))  \\
\p{\vect(L^{l,h})}{\vect(c^{l,h})} &= \frac{\sigma_{c}}{\sqrt{2d_{l-1}}} \hat{A} \concat((\mathbf{1}_{n} \otimes G^{T}), (G^{T} \otimes \mathbf{1}_{n}))  \text{ with $\hat{A}:=\diag(\vect(A^{T})$} \\
\intertext{derived using $\vect(L^{l,h})=\frac{\sigma_{c}}{\sqrt{2d_{l-1}}}\hat{A} \left((\mathbf{1}_{n} \otimes G^{T})c_{1}^{l,h}+ (G^{T} \otimes \mathbf{1}_{n}^{T})c_{2}^{l,h} \right)$}
\intertext{$\frac{\sigma_{c}}{\sqrt{2d_{l-1}}}  \hat{A} \left((\mathbf{1}_{n} \otimes G^{T}c_{1}^{l,h})+  (G^{T}c_{2}^{l,h}  \otimes \mathbf{1}_{n}) \right) =\frac{\sigma_{c}}{\sqrt{2d_{l-1}}}\hat{A} \concat(\mathbf{1}_{n} \otimes I_{n}, I_{n} \otimes  \mathbf{1}_{n}) \concat(G, G)^{T} $}
\p{\vect(L^{l, h})}{\vect(G^{h-1})} &=  \frac{\sigma_{c}}{\sqrt{2d_{l-1}}} \hat{A} \left((\mathbf{1}_{n} \otimes (I_{n} \otimes c_{1}^{l,h^{T}} ) + (I_{n} \otimes c_{2}^{l,h^{T}}) \otimes \mathbf{1}_{n} )  \right) \\
\p{\vect(F^{l})}{\vect(G^{l-1})} = &\frac{1}{\sqrt{H}}\sum_{h}^{H} 
 \biggr[ \left( I_{n} \otimes \left(\frac{\sigma_{w}}{\sqrt{d_{l-1}}}W^{l,h} G^{l-1} \right) \right)  \SSigma{l,h} \p{\vect(L^{l, h})}{\vect(G^{h-1})} + \left( \sigma_{1}(L^{l,h})^{T} \otimes \frac{\sigma_{w}}{\sqrt{d_{l-1}}}W^{l,h} \right)   \biggl]
\end{align*}
\subsubsection*{Proof Strategy for GAT* GP \& NTK}
The proof of the GAT* Gaussian Process follows from \citep[Theorem 1]{pmlr-v119-hron20a} which
is the Scaled Attention Theorem from \citep{yang1}. Theorem 1/ Scaled Attention Theorem is based on \citep[Mastertheorem]{yang1, yang2} which proves that the infinite width Neural Network connections to GP and NTK hold for a variety of architectures.
The authors of \cite{pmlr-v119-hron20a} generalizes Theorem 1 to a particular scaling, namely the $d^{-\frac{1}{2}}$ scaling for Attention Neural Networks.
This is made precise in Theorem 3 from the aforementioned authors.
Adapting Theorem 1 is enough to prove the GAT* Gaussian Process derivation because the GAT* Model does not require the $d^{-\frac{1}{2}}$ scaling. Contrary to the previous proofs in this work which derived the infinite Width Limit for $d_{l-1} \rightarrow \infty$ layer after layer, the proofs for the NTK and GP for the GAT* are conducted for
min$\{H,d_{l-1}\} \to \infty$ (i.e. the attention heads and weight widths go to infinity simultaneously for each layer). Similar techniques have been used by \citep{gp1} to demonstrate that infinite width Neural Nets are Gaussian Processes when all widths go to infinity simultaneously.
Before we start, we have to show that our GAT* model meets all assumptions from \cite[Theorem 1/ Theorem 3]{pmlr-v119-hron20a}.
We have to make sure that the output of the GAT* model
is bounded by a constant which is independent of $H,d_{l-1}$. This will make it possible to use Lemma 32 from \cite{pmlr-v119-hron20a}. This way, all assumptions of Theorem 3 will be met.
To prove boundedness of our GAT* Model it suffices to prove boundedness for the only component that is different to the Attention Neural Network (defined in \cite{pmlr-v119-hron20a}) namely $L$. Boundedness of the Model then follows by induction and Hölder Inequality (see \citep[Lemma 32]{pmlr-v119-hron20a} which is based on \cite{gp1}).
Boundedness of $L$ follows by direct application of  of \citep[Lemma 19]{gp1}. In words: $L_{ij}$ is just an inner product of a constant vector (the data) with a normally distributed vector. 
This inner product can be bounded by a constant independent of the size of the normally distributed vector. Each entry $L_{ij}$ is just the same vector $c$ which will establish the bound on all of $L$ using the mentioned Lemma 19. The proof of the NTK is conducted using similar techniques to the NTK proof from  \citep{pmlr-v119-hron20a}.

\subsubsection{GAT*GP Gaussian Process (Theorem \ref{thmgattt})}\label{gatgpproof}

\begin{lemma}\label{lemmagatgp}
If the output of a GAT Layer $l-1$ is a GP, so $\vect(F^{l-1}) \sim GP(0, \Omega \otimes I_{d_{l-1}})$ for a fixed $\Omega$ and the $d_{l-1} \to \infty$, then  $\vect(L^{l}) \sim  GP( 0,  \psi(\Omega))$, with
\begin{gather*}
    \psi(\Omega) :=  \Exp{\sigma_{1}(u)\sigma_{1}(u^{T})} \text{ with $u \sim GP\bigl(0, 
    \sigma_{c} J_{A}
    \begin{psmallmatrix}
    \Omega & \Omega \\
    \Omega & \Omega \\
    \end{psmallmatrix}
    J_{A}^{T} \otimes I_{d_{l-1}} \bigr))$ }\\
    \text{and $J_{A} := \diag(\vect(A^{T})\concat((\mathbf{1}_{n} \otimes I_{n}), (I_{n} \otimes  \mathbf{1}_{n}))$}
\end{gather*} 
\begin{proof}  \noindent
We will demonstrate the proof for $d_{l-1} \rightarrow \infty$ sequentially for each layer, utilizing Induction to conclude the proof. (The base case follows from using the Definition of $L^{1}$). This will make the proof much simpler.
In the case of min$\{d_{l-1} \} \rightarrow \infty$, (so all Widths going to infinity simultaneously) the proof can be conducted applying \cite[Mastertheorem/ Scaled Attention Neural Networks]{yang1} or \cite[Theorem 1/ Theorem 3, Part I)]{pmlr-v119-hron20a}.
The proof is conducted for $L^{l,h}$ so for every head $h$ with corresponding $c^{l,h}$ but we will omit $h$ and write $c^{l}$ instead $c^{l,h}$ and $L^{l}$ instead $L^{l}$.
\begin{align*}
\Exp{L^{l}_{ij}} &= 0 \\
\Exp{L^{l}_{lm}L^{l}_{st}} &= \frac{\sigma_{c}^{2}}{2d_{l-1}} A_{lm}A_{st} \Exp{\left(c_{1}^{l^{T}} (G_{\cdot l} + G_{\cdot m}) + c_{2}^{l^{T}} (G_{\cdot l} + G_{\cdot m})\right) +  \left(c_{1}^{l^{T}} (G_{\cdot s} + G_{\cdot t}) c_{2}^{l^{T}} (G_{\cdot s} + G_{\cdot t})\right)  } = \\
&= \frac{\sigma_{c}^{2}}{2d_{l-1}} A_{lm}A_{st} 
\biggl( \Exp{c_{1}^{l^{T}}(G_{\cdot l} + G_{\cdot m})c_{1}^{l^{T}}(G_{\cdot s} + G_{\cdot t})}
+ \Exp{c_{1}^{l^{T}}(G_{\cdot l} + G_{\cdot m})c_{2}^{l^{T}}(G_{\cdot s} + G_{\cdot t})} \\
&+ \Exp{c_{2}^{l^{T}}(G_{\cdot l} + G_{\cdot m})c_{1}^{l^{T}}(G_{\cdot s} + G_{\cdot t})}
+ \Exp{c_{2}^{l^{T}}(G_{\cdot l} + G_{\cdot m}) c_{2}^{l^{T}}G_{\cdot l}{s} + G_{\cdot t})} \biggr) \\
&=\frac{\sigma_{c}^{2}}{2d_{l-1}} A_{lm}A_{st} \biggl(
 \Exp{c_{1}^{l^{T}}(G_{\cdot l} + G_{\cdot m})c_{1}^{l^{T}}(G_{\cdot s} + G_{\cdot t})} 
 + \Exp{c_{2}^{l^{T}}(G_{\cdot l} + G_{\cdot m}) c_{2}^{l^{T}}G_{\cdot l}{s} + G_{\cdot t})}
\biggr) \\
&\hspace*{-1.5cm} =\frac{\sigma_{c}^{2}}{2d_{l-1}} A_{lm}A_{st} \biggl( 
\Exp{\sum_{uv}  c_{1u}^{l}c_{1v}^{l}G_{ul}G_{vs}} 
+\Exp{\sum_{uv}  c_{1u}^{l}c_{1v}^{l}G_{ul}G_{vt}}
+\Exp{\sum_{uv}  c_{1u}^{l}c_{1v}^{l}G_{um}G_{vs}} 
+\Exp{\sum_{uv}  c_{1u}^{l}c_{1v}^{l}G_{um}G_{vt}} \\
&\hspace*{-1.5cm}+\Exp{\sum_{uv}  c_{2u}^{l}c_{2v}^{l}G_{ul}G_{vs}}
+\Exp{\sum_{uv}  c_{2u}^{l}c_{2v}^{l}G_{ul}G_{vt}}
+\Exp{\sum_{uv}  c_{2u}^{l}c_{2v}^{l}G_{um}G_{vs}}
+\Exp{\sum_{uv}  c_{2u}^{l}c_{2v}^{l}G_{um}G_{vt}}
\biggr) \\
\intertext{\centering $\overset{\text{Converges in probability for $d_{l-1} \rightarrow \infty$}}{\longrightarrow}$}
&\sigma_{c}^{2} A_{lm}A_{st}\biggl( 
\Exp{\sigma_{1}(u_{l})\sigma_{1}(u_{s})}
+ \Exp{\sigma_{1}(u_{l})\sigma_{1}(u_{t})}
+\Exp{\sigma_{1}(u_{m})\sigma_{1}(u_{s})}
+\Exp{\sigma_{1}(u_{m})\sigma_{1}(u_{t})}
\biggr)
\end{align*}
with $u=\vect(F^{l-1}) \sim GP(0, \Omega \otimes I_{d_{l-1}})$
\end{proof}
\end{lemma} \noindent

\thmgattt*
\begin{sproof}
Theorem \ref{GAT-NTK123} can be derived by just applying \cite[Part II) Theorem 1 (which itself is from \cite{yang2} or Part II) Theorem 3]{pmlr-v119-hron20a}. Notice that for proving Part II) Theorem 3) at no point is the actual definition of the Attention Neural Networks used, therefore the proves applies to our model without any adaption.
\begin{align*}
        \Exp{\vect(F^{l})} &= 0 \\
         \Exp{\vect(F^{l})\vect(F^{l})^{T}} &= 
         \Exp{ \left( \sigma_{1}(L^{l,h}) \otimes \frac{\sigma_{w}}{\sqrt{d_{l}}} W^{l,h} \right)
    G^{l-1^{T}}G^{l-1} \left( \sigma_{1}(L^{l,h}) \otimes \frac{\sigma_{w}}{\sqrt{d_{l}}}W^{l,h} \right)^{T}} \\
    \intertext{for block indices $I,J$}
    \Exp{\vect(F^{l})\vect(F^{l})^{T}}_{IJ} &= \\
\Exp{\frac{1}{H}\sum_{h,h'}^{H}\sum_{ST} \left( \sigma_{1}(L^{l,h}) \otimes I_{d_{l}} \right)_{IS}& \left[ \left( I_{n} \otimes \frac{\sigma_{w}}{\sqrt{d_{l}}} W^{l,h} \right)   G^{l-1^{T}}G^{l-1}
\left( I_{d_{l}} \otimes \frac{\sigma_{w}}{\sqrt{d_{l}}}W^{l,h^{T}} \right) \right]_{ST}
\left(\sigma_{1}(L^{l,h}) \otimes I_{d_{l}}\right)_{TJ}
}\\
=\Exp{\frac{1}{H}\sum_{h,h'}^{H}\sum_{ST}  \sigma_{1}(L^{l,h})_{is} \frac{\sigma_{w}^{2}}{d_{l}} W^{l,h} & (G^{l-1^{T}}G^{l-1})_{ST} W^{l,h'^{T}}  \sigma_{1}(L^{l,h'})_{tj} }
\intertext{for indices $u,v$ } 
 \Exp{\vect(F^{l})\vect(F^{l})^{T}}_{IJ_{uv}} &=\Exp{\frac{1}{H}\sum_{h,h'}^{H}\sum_{ST} \sum_{lm}  \sigma_{1}(L^{l,h})_{is} \sigma_{1}(L^{l,h'})_{tj} \frac{\sigma_{w}^{2}}{d_{l}} W^{l,h}_{ul}W^{l,h'}_{vm}  (G^{l-1^{T}}G^{l-1})_{ST_{lm}}} \\
 &=\frac{1}{H}\sum_{h,h'}^{H}\sum_{ST} \sum_{lm}  \sigma_{1}(L^{l,h})_{is} \sigma_{1}(L^{l,h'})_{tj} \Exp{\frac{\sigma_{w}^{2}}{d_{l}} W^{l,h}_{ul}W^{l,h'}_{vm}  (G^{l-1^{T}}G^{l-1})_{ST_{lm}}}
 \end{align*}
 \begin{gather*}
 \intertext{Now by a hand wavy argument we first let $d_{l-1} \rightarrow \infty$ to show,}
\Exp{\frac{\sigma_{w}^{2}}{d_{l}} W^{l,h}_{ul}W^{l,h'}_{vm}  (G^{l-1^{T}}G^{l-1})_{ST_{lm}}} \overset{P}{\rightarrow} \sigma_{w}^{2}\Lambda^{l-1}_{ST_{lm}}
\intertext{and then $H \rightarrow \infty$ and by the Law of Large Numbers }
\frac{\sigma_{w}^{2}}{H}\sum_{h,h'}^{H}\sum_{ST} \sum_{lm}  \sigma_{1}(L^{l,h})_{is} \sigma_{1}(L^{l,h'})_{tj} \Lambda_{ST_{lm}}
\overset{P}{\rightarrow}
\sigma_{w}^{2}\Exp{\sum_{ST} \sum_{lm} \sigma_{1}(L^{l,h})_{is} \sigma_{1}(L^{l,h'})_{tj}\Lambda_{ST_{lm}}} \\
\intertext{the righthandside is nothing else than}
\sigma_{w}^{2}\Exp{\sum_{ST} \sum_{lm} \sigma_{1}(L^{l,h})_{is} \sigma_{1}(L^{l,h'})_{tj}\Lambda_{ST_{lm}}}
=
 \sigma_{w}^{2}\sum_{st}^{n}\Exp{\sigma_{1}(L^{l,h})_{is} \sigma_{1}(L^{(l,h')})_{tj}}  \Lambda^{l-1}_{st} 
=
\sigma_{w}^{2} \sum_{st}^{n} \Lambda^{l-1}_{st} \psi(\Lambda^{l-1})_{IJ_{st}} \\
\implies \Exp{\vect(F^{l})\vect(F^{l})^{T}} \rightarrow \bm( \sigma_{w}^{2} \Lambda^{l-1}, \psi(\Lambda^{l-1}))
\end{gather*}
For a rigourous treatement and the case of min$\{H, d_{l-1} \} \rightarrow \infty$ the proof can be concluded using \citep[Mastertheorem]{yang1} or \citep[Theorem 3, Part II)]{pmlr-v119-hron20a}.
\end{sproof}

\subsubsection{Neural Tangent Kernel (Theorem~\ref{thm:gat2})} \label{gatntkproof}
\thmgatt*
\begin{proof}\renewcommand{\qedsymbol}{}
As a reminder the NTK of depth $l$ is
\begin{gather*}
    \p{\vect(F^{l}(\theta, X))}{\theta} \left( \p{\vect(F^{l}(\theta, X))}{\theta} \right)^{T}
\end{gather*}
for parameters $\theta$. We  when the width of the hidden layers goes to infinity this expression converges in probability to $\Theta^{l} \otimes I_{d_{l}}$. Now to the proof, first realize that:
\begin{gather*}
   \p{\vect(F^{l}(\theta, X))}{\theta} \left( \p{\vect(F^{l}(\theta, X))}{\theta} \right)^{T} =   \pv{F^{l}}{W^{l}} \left( \pv{F^{l}}{W^{l}} \right)^{T} + \frac{1}{H} \sum_{h}^{H}
\pv{F^{l}}{c^{l,h}} \left( \pv{F^{l}}{c^{l,h}}\right)^{T} \\ 
+ \pv{F^{l}}{G^{l-1}}\pv{G^{l-1}}{F^{l-1}} \p{\vect(F^{l-1}(\theta, X))}{\theta} \left( \p{\vect(F^{l-1}(\theta, X))}{\theta} \right)^{T}  \left(\pv{F^{l}}{G^{l-1}}   \pv{G^{l-1}}{F^{l-1}} \right)^{T} 
\end{gather*}
Each of the terms is simplified and probability in convergence is proven for min$\{d_{l-1}, H\} \to \infty$.
\end{proof}
\begin{lemma}
\begin{align*}
        \pv{F^{l}}{W^{l}} \left( \pv{F^{l}}{W^{l}} \right)^{T} \overset{P}{\longrightarrow} \bm\left(\sigma_{w}^{2} \Lambda^{l-1}, \psi(\Lambda^{l-1})  \right) \otimes I_{d_{l}}
\end{align*}

\begin{proof}
\begin{gather*}
        \pv{F^{l}}{W^{l}} \left( \pv{F^{l}}{W^{l}} \right)^{T} = \frac{\sigma_{w}^{2}}{d_{l-1}} \left[F^{l1}, ...,  F^{lH} \right] \left[F^{l1}, ...,  F^{lH} \right]^{T} \otimes I_{d_{l}} \\
        \frac{\sigma_{w}^{2}}{Hd_{l-1}} \left( \left[F^{l1}, ...,  F^{lH} \right] \left[F^{l1}, ...,  F^{lH} \right]^{T} \right)= \sigma_{w}^{2} \sum_{h}^{H} \frac{1}{Hd_{l-1}} \sigma_{1}(L^{l,h}) G^{l-1^{T}}G^{l-1}  \sigma_{1}(L^{l,h}) \\
        \intertext{We are going to focus on elements $i,j$ and rewrite it}
 \left(\pv{F^{l}}{W^{l}} \left( \pv{F^{l}}{W^{l}} \right)^{T} \right)_{ij} =       
 \frac{\sigma_{w}^{2}}{H} \sum_{h}^{H} \sum_{l,s}^{n} \sigma_{1}(L^{l,h})_{is} \sigma_{1}(L^{l,h})_{tj}  \frac{\left(G^{l-1^{T}}G^{l-1} \right)_{st} }{d_{l-1}} \\
 \overset{P}{\longrightarrow} \sigma_{w}^{2} \sum_{st}^{n} \Exp{ \sigma_{1}(L^{l,h})_{is} \sigma_{1}(L^{l,h})_{tj}} \Lambda^{l-1}_{st} = \sigma_{w}^{2} \sum_{st}^{n}  \Lambda^{l-1}_{st} \psi(\Lambda^{l-1})_{IJ_{st}}\\
\end{gather*}
To prove convergence in probability (last step). for for min$\{d_{l-1}, H\} \rightarrow \infty$ one can apply \cite[Lemma 19]{pmlr-v119-hron20a} or \cite[Mastertheorem]{yang2}.
\end{proof}
\end{lemma}

\begin{lemma}\label{ref0}
\begin{gather*}
 \sum_{h}^{H}\pv{F^{l}}{c^{l,h}} \left( \pv{F^{l}}{c^{l,h}}\right)^{T}  \overset{P}{\longrightarrow}
 \bm \left[\sigma_{w}^{2} \Lambda^{l-1} ,   \sigma_{c}^{2} \left( \gamma_{A}(\Lambda^{l-1}) \odot  \dot{\psi}(\Lambda^{l-1}) \right)  \right] \otimes I_{d_{l}}
 \end{gather*}
 \end{lemma}

\begin{proof}
\begin{gather*}
    \frac{1}{H}\sum_{h}^{H}\left(
    \pv{F^{l}}{\sigma_{1}(L^{l, h})}
    \pv{\sigma_{1}(L^{l, h}}{L^{l, h}}  
    \pv{L^{l, h}}{c^{l, h}} 
\right) 
\left(
    \pv{F^{l}}{\sigma_{1}(L^{l, h})}  
    \pv{\sigma_{1}(L^{l, h}}{L^{l, h}}  
    \pv{L^{l, h}}{c^{l, h}}  
\right)^{T} = \\
\frac{1}{H}\sum_{h}^{H} \left(
    \left( 
        I_{n} \otimes \frac{\sigma_{w}}{\sqrt{d_{l-1}}} W^{l,h}G^{l} \right) \SSigma{l,h} \pv{L^{l, h}}{c^{l, h}}  \left(\pv{L^{l, h}}{c^{l, h}} \right)^{T} \SSigma{l,h}
    \left(I_{n} \otimes \frac{\sigma_{w}}{\sqrt{d_{l-1}}} G^{l^{T}}W^{l,h^{T}} 
\right)
\right) \\
\intertext{Focusing on block indices $I,J$}
 \left(  \pv{F^{l}}{c^{l,h}} \left( \pv{F^{l}}{c^{l,h}}\right)^{T} \right)_{IJ} =  \frac{\sigma_{w}^{2}}{Hd_{l-1}} \sum_{h}^{H} 
      W^{l,h}G^{l}  \SSigma{l,h}_{II} \biggl[ \pv{L^{l, h}}{c^{l, h}}  \left(\pv{L^{l, h}}{c^{l, h}} \right)^{T} \biggr]_{IJ}  \SSigma{l,h}_{JJ}
     G^{l-1^{T}}W^{l,h^{T}} \\
\intertext{Focusing on elements with indices $u,v$ of $I,J$}  
\hspace*{-1.5cm} \left( \pv{F^{l}}{c^{l,h}} \left( \pv{F^{l}}{c^{l,h}}\right)^{T} \right)_{IJ_{uv}} = 
\sigma^{2}_{w}  \frac{1}{Hd_{l-1}}\sum_{h}^{H} \sum_{st}^{n} \biggr[\pv{L^{l, h}}{c^{l, h}}  \left(\pv{L^{l, h}}{c^{l, h}} \right)^{T} \biggl]_{IJ_{st}}  \langle W^{l,h}_{u}, G^{l-1}_{ \cdot s} \rangle \langle W^{l,h}_{v}, G^{l-1}_{ \cdot t} \rangle \SSigma{l,h}_{II_{ss}} \SSigma{l,h}_{JJ_{tt}} \\
 =\sigma^{2}_{w} \frac{1}{H}\sum_{h}^{H} \sum_{st}^{n}  \biggr[\pv{L^{l, h}}{c^{l, h}}  \left(\pv{L^{l, h}}{c^{l, h}} \right)^{T} \biggl]_{IJ_{st}} \SSigma{l,h}_{II_{ss}} \SSigma{l,h}_{JJ_{tt}} \frac{1}{d_{l-1}}\sum_{ab}^{d_{l-1}} W^{l,h}_{ua}W^{l,h}_{vb} G^{l-1}_{as} G^{l-1}_{bt} 
 \intertext{as a reminder}
 \biggr[\pv{L^{l, h}}{c^{l, h}}  \left(\pv{L^{l, h}}{c^{l, h}} \right)^{T} \biggl]_{IJ_{st}} = \frac{\sigma_{c}^{2}}{2d_{l-1}}\langle \hat{A} \concat((\mathbf{1}_{n} \otimes G^{T}), (G^{T} \otimes \mathbf{1}_{n}))_{I_{s}}, \hat{A} \concat((\mathbf{1}_{n} \otimes G^{T}), (G^{T} \otimes \mathbf{1}_{n}))_{J_{t}}   \rangle \\
 = \sigma_{c}^{2}\hat{A}_{II_{ss}} \hat{A}_{JJ_{tt}} 
 \frac{\langle G_{\cdot I_{s}}, G_{\cdot J_{t}} \rangle +   \langle G_{\cdot I_{s}}, G_{\cdot J_{t}} \rangle}{2d_{l-1}} 
 \\
 \\
 = \sigma^{2}_{w}\sigma_{c}^{2} \frac{1}{H}\sum_{h}^{H} \sum_{st}^{n} 
\hat{A}_{II_{ss}} \hat{A}_{JJ_{tt}} 
 \frac{\langle G_{\cdot I_{s}}, G_{\cdot J_{t}} \rangle +   \langle G_{\cdot I_{s}}, G_{\cdot J_{t}} \rangle}{2d_{l-1}} \SSigma{l,h}_{II_{ss}} \SSigma{l,h}_{JJ_{tt}} \frac{1}{d_{l-1}}\sum_{ab}^{d_{l-1}} W^{l,h}_{ua}W^{l,h}_{vb} G^{l-1}_{as} G^{l-1}_{bt}
 \intertext{converging in probability for $\text{min}\{d_{l-1},H\} \rightarrow \infty$}
= \delta_{uv} \sigma^{2}_{w} \sum_{st}^{n}   (J_{A}  \begin{psmallmatrix}
    \Lambda^{l-1} & \Lambda^{l-1} \\
    \Lambda^{l-1} & \Lambda^{l-1} \\
\end{psmallmatrix}  J_{A}^{T})_{IJ_{st}}  \dot{\psi}(\Lambda^{l-1})_{IJ_{st}} \dot{\Lambda}^{l-1}_{st} 
\intertext{using our previously defined shorthand definition $\gamma_{A}(\Omega) := J_{A}  \begin{psmallmatrix}
    \Omega & \Omega \\
    \Omega & \Omega \\
\end{psmallmatrix}  J_{A}^{T}$ we end up with}
= \delta_{uv} \sigma^{2}_{w} \sum_{st}^{n}   \gamma_{A}(\Lambda^{l-1})_{IJ_{st}}  \dot{\psi}(\Lambda^{l-1})_{IJ_{st}} \Lambda^{l-1}_{st} 
\end{gather*}
Similar to the previous lemma to conclude convergence in probability for  $\text{min}\{d_{l-1},H\} \to \infty $ one can use \cite[Lemma 21]{pmlr-v119-hron20a} or \cite[Master Theorem]{yang2}.
\end{proof}

\begin{lemma} \label{ref1}
    \begin{gather*}
\pv{F^{l}}{G^{l-1}}  \pv{G^{l-1}}{F^{l-1}}  \left( \Theta^{l-1}  \otimes I_{d_{l-1}} \right) \left(\pv{F^{l}}{G^{l-1}}   \pv{G^{l-1}}{F^{l-1}} \right)^{T}  \\
\overset{P}{\longrightarrow}  \\
\bm \left[ \sigma_{w}^{2} \Lambda^{l-1}, \sigma^{2}_{c} \left(\gamma_{A} (\Theta^{l-1} \odot \dot{\Lambda}^{l-1}) \odot  \dot{\psi}(\Lambda^{l-1}) \right) \right]
+ \bm \left[ \sigma_{w}^{2} \left(
            \Theta^{l-1} \odot \dot{\Lambda}^{l-1}\right), \psi(\Lambda^{l-1}) \right]
\end{gather*}
\end{lemma}
\begin{proof}
\begin{gather*}
\left( \pv{F^{l}}{G^{l-1}}  \pv{G^{l-1}}{F^{l-1}} \right) \left( \Theta^{l-1}  \otimes I_{d_{l-1}} \right)               \left(\pv{F^{l}}{G^{l-1}}   \pv{G^{l-1}}{F^{l-1}} \right)^{T} = \\
\pv{F^{l}}{G^{l-1}}   \hat{\Theta}^{l}   \pv{F^{l}}{G^{l-1}}^{T}
\intertext{$\hat{\Theta}^{l}_{IJ} := \Theta^{l-1}_{ij}  \SSigma{l,h}_{II} \SSigma{l,h'}_{JJ}$, similar to \cite[Section B2.2 Indirect Contributions]{pmlr-v119-hron20a}. Also note that for $d_{l-1}\rightarrow \infty $, 
$\hat{\Theta}^{l} \overset{P}{\rightarrow}  (\Theta^{l - 1} \odot \dot{\Lambda}^{l-1}) \otimes I_{d_{l-1}}$.
(see for example \cite[Master Theorem]{yang1}). Now continuing with the simplification.}
 \frac{1}{H}\sum_{h, h'}^{H} \biggr[ \left( I_{n} \otimes \left(\frac{\sigma_{w}}{\sqrt{d_{l}}}W^{l,h} G^{l-1} \right) \right)  \SSigma{l,h} \p{\vect(L^{l, h})}{\vect(G^{h-1})} + \left( \sigma_{1}(L^{l,h}) \otimes \frac{\sigma_{w}}{\sqrt{d_{l}}}W^{l,h} \right)   \biggl]  \\ 
 \cdot \  \hat{\Theta}^{l} 
 \biggr[ \left( I_{n} \otimes \left(\frac{\sigma_{w}}{\sqrt{d_{l}}}W^{l,h'} G^{l-1} \right) \right)  \SSigma{l,h'} \p{\vect(L^{l, h})}{\vect(G^{h-1})} + \left( \sigma_{1}(L^{l,h'}) \otimes \frac{\sigma_{w}}{\sqrt{d_{l-1}}}W^{l,h'} \right)   \biggl]^{T} \\
=\frac{1}{H}\sum_{h, h'}^{H} \left( A^{l,h,h'} + B^{l,h,h'} + C^{l,h,h'} + C^{l,h,h'^{T}}\right) \\
\end{gather*}
Now we will simplify this final expression using the following lemmas to conclude this proof.
\end{proof}

\begin{lemma}
\begin{gather*}
        \frac{1}{H} \sum_{h,h'}^{H} A^{l,h,h'} \overset{P}{\rightarrow} \bm \left[ \sigma_{w}^{2} \Lambda^{l-1}, \sigma^{2}_{c} \left( \gamma_{A}(\Theta^{l-1} \odot \dot{\Lambda}^{l-1}) \odot  \dot{\psi}(\Lambda^{l-1}) \right) \right] \otimes I_{d_{l}}
\end{gather*}
\end{lemma}
\begin{proof}
    \begin{gather*}
A^{l,h,h'} = \biggl[ \left( I_{n} \otimes \left(\frac{\sigma_{w}}{\sqrt{d_{l}}}W^{l,h} G^{l-1} \right) \right)  \SSigma{l,h} \p{\vect(L^{l, h})}{\vect(G^{h-1})} \biggr] \hat{\Theta}^{l}   
\biggl[ \left( I_{n} \otimes \left(\frac{\sigma_{w}}{\sqrt{d_{l-1}}}W^{l,h'} G^{l-1} \right) \right)  \SSigma{l,h'} \p{\vect(L^{l, h})}{\vect(G^{h-1})} \biggr]^{T} \\
= \left( I_{n} \otimes \left(\frac{\sigma_{w}}{\sqrt{d_{l}}}W^{l,h} G^{l-1} \right) \right)\SSigma{l,h}
\p{\vect(L^{l, h})}{\vect(G^{h-1})} \hat{\Theta}^{l}
\left( \p{\vect(L^{l, h})}{\vect(G^{h-1})} \right)^{T} \SSigma{l,h'} \left( I_{n} \otimes \left(\frac{\sigma_{w}}{\sqrt{d_{l-1}}} G^{l-1^{T}} W^{l,h'^{T}} \right) \right)
\intertext{focusing on block indices $I,J$}
A^{l,h,h'}_{IJ} = \frac{\sigma_{w}^{2}}{d_{l-1}}W^{l,h} G^{l-1} \SSigma{l,h}_{II} \left[\p{\vect(L^{l, h})}{\vect(G^{h-1})} \hat{\Theta}^{l}\left( \p{\vect(L^{l, h})}{\vect(G^{h-1})} \right)^{T}\right]_{IJ}\SSigma{l,h'}_{JJ}G^{l-1^{T}} W^{l,h'^{T}} \\
 \left[\p{\vect(L^{l, h})}{\vect(G^{h-1})} \hat{\Theta}^{l}\left( \p{\vect(L^{l, h})}{\vect(G^{h-1})} \right)^{T}\right]_{IJ_{st}} = \frac{\sigma_{c}^{2}}{2d_{l-1}}
 \Tilde{A}_{II_{ss}}  \Tilde{A}_{JJ_{tt}}  \left( c_{1}^{T} \hat{\Theta}^{l}_{st} c_{1} +  c_{1}^{T} \hat{\Theta}^{l}_{sJ} c_{2}  +               c_{2}^{T} \hat{\Theta}^{l}_{It} c_{1} +  c_{2}^{T} \hat{\Theta}^{l}_{IJ} c_{2} \right) 
 \intertext{with $\hat{\Theta}^{l}_{st} \inR{d_{l-1} \times d_{l-1}}$.Focusing on element indices $u,v$ }
A^{l,h,h'}_{IJ_{uv}} = \frac{\sigma^{2}_{w}}{d_{l-1}} \sum_{st} \left[\p{\vect(L^{l, h})}{\vect(G^{h-1})} \hat{\Theta}^{l}\left( \p{\vect(L^{l, h})}{\vect(G^{h-1})} \right)^{T}\right]_{IJ_{st}}   \SSigma{l,h}_{II_{ss}} \SSigma{l,h'}_{JJ_{tt}} \langle W_{u}, G^{l-1}_{ \cdot s} \rangle \langle W_{v}, G^{l-1}_{ \cdot t} \rangle \\
A^{l,h,h'}_{IJ_{uv}} =\frac{\sigma^{2}_{w}\sigma^{2}_{c}}{2d_{l-1}^{2}} \sum_{st}  \Tilde{A}_{II_{ss}}  \Tilde{A}_{JJ_{tt}}  \left( c_{1}^{T} \hat{\Theta}^{l}_{st} c_{1} +  c_{1}^{T} \hat{\Theta}^{l}_{sJ} c_{2}  +               c_{2}^{T} \hat{\Theta}^{l}_{It} c_{1} +  c_{2}^{T} \hat{\Theta}^{l}_{IJ} c_{2} \right) 
\SSigma{l,h}_{II_{ss}} \SSigma{l,h'}_{JJ_{tt}} \langle W_{u}, G^{l-1}_{ \cdot s} \rangle \langle W_{v}, G^{l-1}_{ \cdot t} \rangle 
\intertext{follows the previous ideas and the the fact that $\hat{\Theta}^{l} \overset{P}{\rightarrow}  (\Theta^{l - 1} \odot \dot{\Lambda}^{l-1}) \otimes I_{d_{l-1}}$ as in \ref{ref1}}
    \intertext{Convergencence in probability for  $\text{min}\{d_{l-1},H\} \rightarrow \infty$}
    \frac{1}{H} \sum_{h,h'}A^{l,h,h'}_{IJ_{uv}} \overset{P}{\longrightarrow}  \delta_{uv} \delta_{hh'}  \sum_{st}   \Tilde{A}_{II_{ss}}  \Tilde{A}_{JJ_{tt}} 
\dot{\Lambda}^{l-1}_{IJ_{st}} \Theta^{l-1}_{st}\Theta^{l-1}_{ij} \Lambda^{l-1}_{st}
\end{gather*}
which can be conducted using \cite[Master Theorem]{yang2}  or \cite[Lemma 23]{pmlr-v119-hron20a}.
\end{proof}

\begin{lemma}
\begin{gather*}
\frac{1}{H} \sum_{h,h'}^{H} B^{l,h,h'} \overset{P}{\rightarrow} \bm \left[ \sigma_{w}^{2} \left(
            \Theta^{l-1} \odot \dot{\Lambda}^{l-1}\right), \psi(\Lambda^{l-1}) \right]   \otimes I_{d_{l-1}}
\end{gather*}

\begin{proof}
\begin{gather*}
    B^{l,h,h'} = \left( \sigma_{1}(L^{l,h}) \otimes \frac{\sigma_{w}}{\sqrt{d_{l}}} W^{l,h} \right)  
    \hat{\Theta}^{l} \left( \sigma_{1}(L^{l,h'}) \otimes \frac{\sigma_{w}}{\sqrt{d_{l}}}W^{l,h'} \right)^{T} \\
B^{l,h,h'}_{IJ} = \sum_{ST} \left( \sigma_{1}(L^{l,h}) \otimes I_{d_{l}} \right)_{IS} \left[ \left( I_{n} \otimes \frac{\sigma_{w}}{\sqrt{d_{l}}} W^{l,h} \right)   
    \hat{\Theta}^{l} \left( I_{d_{n}} \otimes \frac{\sigma_{w}}{\sqrt{d_{l}}}W^{l,h'^{T}} \right) \right]_{ST}
\left( \sigma_{1}(L^{l,h'}) \otimes I_{d_{l}} \right)_{TJ} \\
B^{l,h,h'}_{IJ} = \sum_{ST}  \sigma_{1}(L^{l,h})_{is}  \frac{\sigma_{w}^{2}}{d_{l}} W^{l,h'}  \hat{\Theta}^{l}_{ST} W^{l,h'^{T}} \sigma_{1}(L^{l,h'})_{tj} \\
\intertext{finally proof in convergence}
\frac{1}{H}\sum_{h}^{H} B^{l,h,h'}_{IJ} \overset{P}{\rightarrow}  \sum_{st}^{n} \Exp{\sigma_{1}(L^{l,h})_{is} \sigma_{1}(L^{l,h'})_{tj}}  \Theta^{l-1}_{st} \dot{\Lambda}^{l-1}_{st} = \sum_{st}^{n} \psi(\Lambda^{l-1})_{IJ_{st}} \Theta^{l-1}_{st} \dot{\Lambda}^{l-1}_{st}
\end{gather*}
To proof convergence in probability for min$\{d_{l-1},H\} \to \infty$
on can use \cite[Master Theorem]{yang2}  or \cite[Lemma 22]{pmlr-v119-hron20a}.
\end{proof}
\end{lemma}

\begin{lemma}
\begin{gather*}
    \frac{1}{H} \sum_{h,h'}^{H} C^{l,h,h'} \overset{P}{\rightarrow} 0
\end{gather*}
\end{lemma}
    \begin{proof}
        \begin{gather*}
\frac{1}{H}\sum_{h, h'}^{H} C^{l,h,h'} = \frac{1}{H}\sum_{h, h'}^{H} \left( I_{n} \otimes \left(\frac{\sigma_{w}}{\sqrt{2d_{l}}}W^{l,h} G^{l-1} \right) \right)  \SSigma{l,h} \p{\vect(L^{l, h})}{\vect(G^{h-1})}    \hat{\Theta}^{l} 
 \left( \sigma_{1}(L^{l,h'}) \otimes \frac{\sigma_{w}}{\sqrt{d_{l-1}}}W^{l,h'} \right)^{T} \\
\intertext{focusing on block indices $IJ$}
 \frac{1}{H}\sum_{h, h'}^{H}C^{l,h,h'}_{IJ} = \frac{1}{H}\sum_{h, h'}^{H} \frac{\sigma_{w}}{\sqrt{2d_{l}}}W^{l,h} G^{l-1}\SSigma{l,h}_{II}  \left( \p{\vect(L^{l, h})}{\vect(G^{h-1})}  \hat{\Theta}^{l} \right)_{I}   \left( ( \sigma_{1}(L^{l,h'})_{j})^{T} \otimes \frac{\sigma_{w}}{\sqrt{d_{l-1}}}W^{l,h'^{T}} \right) \\
 \intertext{focusing on indices $uv$}
 \frac{1}{H}\sum_{h, h'}^{H}C^{l,h,h'}_{IJ_{uv}} = \frac{\sigma_{w}^{2}\sigma_{c}}{\sqrt{2}Hd_{l-1}^{\frac{3}{2}}}\sum_{h, h'}^{H}  \sum_{st} W_{u}^{l,h} G^{l-1}_{\cdot s} \SSigma{l,h}_{II_{ss}} \Tilde{A}_{II_{ss}}  ( c_{1}^{T} \hat{\Theta}^{l}_{s} + c_{2}^{T}\hat{\Theta}^{l}_{I} )  \left( ( \sigma_{1}(L^{l,h'})_{j})^{T} \otimes \frac{\sigma_{w}}{\sqrt{d_{l-1}}}(W^{l,h'}_{v})^{T} \right) \\
 =\frac{\sigma_{w}^{2}\sigma_{c}}{\sqrt{2}Hd_{l-1}^{\frac{3}{2}}}\sum_{h, h'}^{H} 
 \sum_{stz}  \sigma_{1}(L^{l,h'})_{[j, \lfloor t \bmod n \rfloor + 1]} W_{uz}^{l,h} G^{l-1}_{zs}
 \SSigma{l,h}_{II_{ss}} \Tilde{A}_{II_{ss}}
 ( c_{1}^{T} \hat{\Theta}^{l}_{st}+ c_{2}^{T}\hat{\Theta}^{l}_{It} )
 W^{l,h'^{T}}_{v}
 \intertext{Define $\nu(t):=\lfloor t \bmod n \rfloor + 1$}
 \frac{\sigma_{w}^{2}\sigma_{c}}{\sqrt{2}Hd_{l-1}^{\frac{3}{2}}}\sum_{h, h'}^{H} \sum_{stzlm}\sigma_{1}(L^{l,h'})_{j\nu(j)}W_{uz}^{l,h} G^{l-1}_{zs}
 \SSigma{l,h}_{II_{ss}} \Tilde{A}_{II_{ss}}
\left(c_{1l} W^{l,h'}_{vm} \hat{\Theta}^{l}_{st_{lm}} + c_{2l} W^{l,h'}_{vm} \hat{\Theta}^{l}_{It_{lm}} \right)
\end{gather*}
Using Chebyshevs Inequality 
 \begin{gather*}
    \mathbb{P} \left( |S - \mathbb{E}S| \geq \delta \right) \leq \frac{\Exp{S^{2}} -\Exp{S}^{2}}{\delta^{2}}
\end{gather*}
we  show that $\Exp{S^{2}} = \Exp{S^{2}}$ (and  $\Exp{S} =0$) for  min$\{H, d_{l-1}\} \to \infty$
to finish the proof.  \\
The cross terms where Expectation is taken of $c_{1}$ and $c_{2}$ are zero. \\
Defining $\Gamma_{h_{1}h_{2}h'_{1}h'_{2}s_{1}s_{2}} :=  \sigma_{1}(L^{l,h'_{1}})_{j\nu(t_{1})} 
\sigma_{1}(L^{l,h'_{2}})_{j\nu(t_{2})}
\SSigma{l,h_{1}}_{II_{s_{1}s_{1}}} 
\SSigma{l,h_{2}}_{II_{s_{2}s_{2}}}(L^{l,h'_{2}})_{j\nu(t_{2})} \Tilde{A}_{II_{s_{1}s_{1}}}
\Tilde{A}_{II_{s_{2}s_{2}}} $
\begin{gather*}
    \mathbb{E} \biggr[ \bigl(\frac{1}{H}\sum_{h, h'}^{H}C^{l,h,h'}_{IJ_{uv}} \bigr)^{2} \biggl] =\\
    \frac{\sigma_{w}^{4}\sigma_{c}^{2}}{2H^{2}d_{l-1}^{3}} \biggl( \mathbb{E} \biggl[
     \sum_{\substack{h_{1} h_{2} \\ h'_{1} h'_{2} \\  s_{1}  s_{2} \\ t_{1} t_{2} \\ z_{1} z_{2} \\  l_{1} l_{2} \\ m_{1} m_{2} }}
 \Gamma_{h_{1}h_{2}h'_{1}h'_{2}s_{1}s_{2}}
W_{uz_{1}}^{l,h_{1}}
G^{l-1}_{z_{1}s_{1}}
W_{uz_{2}}^{l,h_{2}}
G^{l-1}_{z_{2}s_{2}}
c_{1l_{1}} W^{l,h'_{1}}_{vm_{1}} \hat{\Theta}^{l}_{s_{1}t_{1_{l_{1}m_{1}}}}
c_{1l_{2}} W^{l,h'_{2}}_{vm_{2}} \hat{\Theta}^{l}_{s_{2}t_{2_{l_{2}m_{2}}}}
\biggr] \\
+ 
\mathbb{E} \biggl[
     \sum_{\substack{h_{1} h_{2} \\ h'_{1} h'_{2} \\  s_{1}  s_{2} \\ t_{1} t_{2} \\ z_{1} z_{2} \\  l_{1} l_{2} \\ m_{1} m_{2} }}
 \Gamma_{h_{1}h_{2}h'_{1}h'_{2}s_{1}s_{2}}
W_{uz_{1}}^{l,h_{1}}
G^{l-1}_{z_{1}s_{1}}
W_{uz_{2}}^{l,h_{2}}
G^{l-1}_{z_{2}s_{2}}
c_{2l_{1}} W^{l,h'_{1}}_{vm_{1}} \hat{\Theta}^{l}_{It_{1_{l_{1}m_{1}}}}
c_{2l_{2}} W^{l,h'_{2}}_{vm_{2}} \hat{\Theta}^{l}_{It_{2_{l_{2}m_{2}}}}
\biggr] \biggl)
\end{gather*}
By the boundedness of all components of $\Gamma$  (boundedness in the sense that for growing min$\{H,d_{l-1} \} \rightarrow \infty$ the expression is bounded by a constant not depending on $d_{l-1}$ or $H$) and Hölderlins Inequality it follows that  $\Gamma$ is bounded by a constant only depending polynomially on X. \\

\begin{gather*}
    \frac{\zeta\sigma_{w}^{4}\sigma_{c}^{2}}{2H^{2}d_{l-1}^{3}} \mathbb{E} \biggl[
     \sum_{\substack{h_{1} h_{2} \\ h'_{1} h'_{2} \\  s_{1}  s_{2} \\ t_{1} t_{2} \\ z_{1} z_{2} \\  l_{1} l_{2} \\ m_{1} m_{2} }}
W_{uz_{1}}^{l,h_{1}}
G^{l-1}_{z_{1}s_{1}}
W_{uz_{2}}^{l,h_{2}}
G^{l-1}_{z_{2}s_{2}}
c_{1l_{1}} W^{l,h'_{1}}_{vm_{1}} \hat{\Theta}^{l}_{s_{1}t_{1_{l_{1}m_{1}}}}
c_{1l_{2}} W^{l,h'_{2}}_{vm_{2}} \hat{\Theta}^{l}_{s_{2}t_{2_{l_{2}m_{2}}}}
\biggr] 
\end{gather*}
Note that for min$\{d_{l-1}, H\} \rightarrow \infty, \hat{\Theta}^{l}_{IJ_{lm}} \overset{P}{\longrightarrow} 0$ for $l \neq m$ (from Lemma \ref{ref1}). Therefore
\begin{gather*}
     \frac{\zeta \sigma_{w}^{4}\sigma_{c}^{2}}{2H^{2}d_{l-1}^{3}} \mathbb{E} \biggl[
     \sum_{\substack{h_{1} h_{2} \\ h'_{1} h'_{2} \\  s_{1}  s_{2} \\ t_{1} t_{2} \\ z_{1} z_{2} \\  l \\ m }}
W_{uz_{1}}^{l,h_{1}}
G^{l-1}_{z_{1}s_{1}}
W_{uz_{2}}^{l,h_{2}}
G^{l-1}_{z_{2}s_{2}}
c_{1l} W^{l,h'_{1}}_{vl} \hat{\Theta}^{l}_{s_{1}t_{1_{ll}}}
c_{1m} W^{l,h'_{2}}_{vm} \hat{\Theta}^{l}_{s_{2}t_{2_{mm}}}
\biggr] 
\end{gather*}
The case $u \neq v$
\begin{gather*}
    \frac{\zeta \sigma_{w}^{4}\sigma_{c}^{2}}{2H^{2}d_{l-1}^{3}} \mathbb{E} \biggl[
     \sum_{\substack{h \\ h' \\  s_{1}  s_{2} \\ t_{1} t_{2} \\ z \\  l  }}
W_{uz}^{l,h}
G^{l-1}_{zs_{1}}
W_{uz}^{l,h}
G^{l-1}_{zs_{2}}
c_{1l}c_{1l} W^{l,h'}_{vl}W^{l,h'}_{vl} \hat{\Theta}^{l}_{s_{1}t_{1_{ll}}}
  \hat{\Theta}^{l}_{s_{2}t_{2_{ll}}}
\biggr] \overset{\text{min$\{d_{l-1}, H \} \rightarrow \infty$}}{=} 0
\end{gather*}
For min$\{d_{l-1}, H \} \rightarrow \infty$, $h,h',z,l$ are the indices that go to infinity, but ${H^{2}d_{l-1}^{3}}$ is growing faster than $h,h',z,l$ and therefore the expression is zero. 
The argument can be repeated for the righthandside to conclude for the case $u \neq v$. \\
The case $u = v$
\begin{gather*}
     \frac{\zeta \sigma_{w}^{4}\sigma_{c}^{2}}{2H^{2}d_{l-1}^{3}} \mathbb{E} \biggl[
     \sum_{\substack{h_{1} h_{2} \\ h'_{1} h'_{2} \\  s_{1}  s_{2} \\ t_{1} t_{2} \\ z_{1} z_{2} \\  l \\ m }}
W_{uz_{1}}^{l,h_{1}}
G^{l-1}_{z_{1}s_{1}}
W_{uz_{2}}^{l,h_{2}}
G^{l-1}_{z_{2}s_{2}}
c_{1l} W^{l,h'_{1}}_{ul} \hat{\Theta}^{l}_{s_{1}t_{1_{ll}}}
c_{1m} W^{l,h'_{2}}_{um} \hat{\Theta}^{l}_{s_{2}t_{2_{mm}}}
\biggr] \\
=  \frac{\zeta \sigma_{w}^{4}\sigma_{c}^{2}}{2H^{2}d_{l-1}^{3}} \mathbb{E} \biggl[
     \sum_{\substack{h_{1} h_{2} \\ h'_{1} h'_{2} \\  s_{1}  s_{2} \\ t_{1} t_{2} \\ z_{1} z_{2} \\  l \\ m }}
W^{l,h'_{1}}_{ul}
W^{l,h'_{2}}_{um}
W_{uz_{1}}^{l,h_{1}}
W_{uz_{2}}^{l,h_{2}}
c_{1l}
c_{1m}
G^{l-1}_{z_{1}s_{1}}
G^{l-1}_{z_{2}s_{2}}
\hat{\Theta}^{l}_{s_{1}t_{1_{ll}}}
\hat{\Theta}^{l}_{s_{2}t_{2_{mm}}}
\biggr]
\end{gather*}
The sum over $h_{1}, h_{2}, h'_{1}, h'_{2}$ can be split over:
\begin{enumerate}
    \item $h'_{1}=h'_{2}$ and $h_{1}=h_{2}$ with $l=m$ and $z_{1}=z_{2}$
    \item $h'_{1}=h_{1}$ and $h'_{2}=h_{2}$ with $l=z_{1}$ and $m=z_{2}$
    \item $h'_{1}=h_{2}$ and $h'_{2}=h_{1}$ with $l=z_{2}$ and $m=z_{1}$
\end{enumerate}
With each case being a sum of which four indices grow to infinity. And therefore again the same argument holds.
\end{proof}
\clearpage

\bibliography{bibliography}

\end{document}